\documentclass{article}
\usepackage{multirow}
\usepackage{subfigure}
\usepackage[utf8]{inputenc}
\usepackage{amsmath, amsthm}
\usepackage{hyperref}
\usepackage{bbm}
\usepackage{url}
\usepackage{amssymb}
\usepackage{wrapfig}
\usepackage[overload]{empheq}
\usepackage{cleveref} 
\usepackage{CJK}
\usepackage{indentfirst}
\usepackage{mathtools}
\usepackage{amsmath}
\usepackage{amssymb}
\usepackage{mathtools}
\usepackage{amsthm}
\usepackage{comment}
\usepackage{color}
\usepackage{algorithm}
\usepackage{algorithmicx}
\usepackage{algpseudocode}
\usepackage{braket}
\usepackage{relsize}
\usepackage{adjustbox}
\usepackage{lmodern}
\theoremstyle{plain}
\usepackage{pifont}
\usepackage{float}
\usepackage{enumitem}
\usepackage{natbib}

\usepackage{microtype}
\usepackage{graphicx}
\usepackage{subfigure}
\usepackage{booktabs} 
\usepackage{tikz}

\usepackage[final]{neurips_2025}

\usepackage[utf8]{inputenc} 
\usepackage[T1]{fontenc}    
\usepackage{hyperref}       
\usepackage{url}            
\usepackage{booktabs}       
\usepackage{amsfonts}       
\usepackage{nicefrac}       
\usepackage{microtype}      
\usepackage{xcolor}         

\theoremstyle{plain}
\newtheorem{theorem}{Theorem}[section]

\newtheorem{lemma}[theorem]{Lemma}

\theoremstyle{definition}
\newtheorem{definition}[theorem]{Definition}
\newtheorem{assumption}[theorem]{Assumption}
\theoremstyle{remark}

\newcommand{\norm}[1]{\left \lVert #1 \right\rVert }

\newcommand{\E}{\mathbb{E}}

\DeclareMathOperator*{\cO}{\mathcal{O}}
\DeclareMathOperator*{\Rd}{\mathbb{R}^d}

\title{Global Convergence for Average Reward Constrained MDPs with Primal-Dual Actor Critic Algorithm}

\author{
  Yang Xu\thanks{Equal contribution.} \\
  Purdue University, USA\\
  \texttt{xu1720@purdue.edu} \\
  \And
  Swetha Ganesh\footnotemark[1] \\
  Purdue University, USA \\
  Indian Institute of Science, Bengaluru, India\\
  \texttt{swethaganesh@iisc.ac.in} \\
  \AND
  Washim Uddin Mondal \\
  Indian Institute of Technology Kanpur \\
  \texttt{wmondal@iitk.ac.in} \\
  \And
  Qinbo Bai \\
  Purdue University, USA \\
  \texttt{bai113@purdue.edu} \\
  \And
  Vaneet Aggarwal \\
  Purdue University, USA \\
  \texttt{vaneet@purdue.edu} \\
}

\begin{document}

\maketitle

\begin{abstract}
This paper investigates infinite-horizon average reward Constrained Markov Decision Processes (CMDPs) under general parametrized policies with smooth and bounded policy gradients. We propose a Primal-Dual Natural Actor-Critic algorithm that adeptly manages constraints while ensuring a high convergence rate. In particular, our algorithm achieves global convergence and constraint violation rates of $\Tilde{\mathcal{O}}(1/\sqrt{T})$ over a horizon of length $T$ when the mixing time, $\tau_{\mathrm{mix}}$, is known to the learner. In absence of knowledge of $\tau_{\mathrm{mix}}$, the achievable rates change to $\tilde{\mathcal{O}}(1/T^{0.5-\epsilon})$ provided that $T \geq \tilde{\cO}\left(\tau_{\mathrm{mix}}^{2/\epsilon}\right)$. Our results match the theoretical lower bound for Markov Decision Processes and establish a new benchmark in the theoretical exploration of average reward CMDPs.
\end{abstract}

\section{Introduction}

Reinforcement Learning (RL) is a paradigm where an agent learns to maximize its average reward in an unknown environment through repeated interactions. RL finds applications in diverse domains such as transportation, communication networks, robotics, and epidemic control \citep{al2019deeppool,haliem2021distributed,panju2021queueing,chen2023option,ling2023cooperating}. Among various RL settings, the infinite horizon average reward setup is of particular significance for modeling long-term objectives in practical scenarios. This setting is critical as it aligns with real-world applications requiring persistent and consistent performance over time.

In many applications, agents must operate under certain constraints. For instance, in transportation networks, delivery times must adhere to specified windows; in communication networks, resource allocations must stay within budget limits. Constrained Markov Decision Processes (CMDPs) effectively incorporate these constraints by introducing a cost function alongside the reward function. In average reward CMDPs, agents aim to maximize the average reward while ensuring that the average cost does not exceed a predefined threshold, making them crucial for scenarios where safety, budget, or resource constraints are paramount.

Finding optimal policies for average reward CMDPs is challenging, especially if the environment is unknown. The efficiency of a solution to the CMDP is measured by the rate at which the average global optimality error and the constraint violation diminish as a function of the length of the horizon $T$. Although many existing works focus on the tabular policies \cite{gattami2021reinforcement,agarwal2022regret,agarwal2022concave,chen2022learning}, these solutions do not apply to real-life scenarios where the state space is large or infinite.

General parametrization offers a useful approach for dealing with such scenarios. It indexes policies via finite-dimensional parameters, which makes them suitable for large state space CMDPs. However, as exhibited in Table 1, the current state-of-the-art algorithms for average reward CMDPs with general parametrization achieve a convergence rate of $\Tilde{\mathcal{O}}(1/T^{1/5})$ \cite{bai2024learning}, which is far from the theoretical lower bound of  $\Omega(1/\sqrt{T})$. This significant gap highlights the need for improved algorithms that can achieve theoretical optimality.

\begin{table*}[t]
    \centering
    \resizebox{1.0\textwidth}{!}
    {
    \begin{tabular}{|c|c|c|c|c|c|}
        \hline
        Algorithm & Global Convergence & Violation & Mixing Time Unknown & Model-free & Setting\\
	\hline
        Algorithm 1 in \citep{chen2022learning} & $\tilde{\mathcal{O}}\bigg(1/\sqrt{T}\bigg)$  & $\tilde{\mathcal{O}}\bigg(1/\sqrt{T}\bigg)$  & No & No & Tabular\\
        \hline
        Algorithm 3 in \citep{chen2022learning} & $\tilde{\mathcal{O}}\bigg(1/T^{1/3}\bigg)$ & $\tilde{\mathcal{O}}\bigg(1/T^{1/3}\bigg)$ & Yes & No & Tabular\\
        \hline
        UC-CURL and PS-CURL \citep{agarwal2022concave} & $\tilde{\mathcal{O}}\bigg(1/\sqrt{T}\bigg)$ & $0$ & Yes & No & Tabular\\
	\hline
        Algorithm 2 in \citep{ghosh2022achieving} & $\tilde{\mathcal{O}}\bigg(1/T^{1/4}\bigg)$ & $\tilde{\mathcal{O}}\bigg(1/T^{1/4}\bigg)$ & Yes & - & Linear MDP\\
	\hline
        Algorithm 3 in \citep{ghosh2022achieving} & $\tilde{\mathcal{O}}\bigg(1/\sqrt{T}\bigg)$ & $\tilde{\mathcal{O}}\bigg(1/\sqrt{T}\bigg)$ & No & - & Linear MDP\\
	\hline
        Triple-QA \citep{wei2022provably} & $\tilde{\mathcal{O}}\bigg(1/T^{1/6}\bigg)$ & $0$ & Yes & Yes & Tabular\\
	\hline
        Algorithm 1 in \citep{bai2024learning} & $\tilde{\mathcal{O}}\bigg(1/T^{1/5}\bigg)$ & $\tilde{\mathcal{O}}\bigg(1/T^{1/5}\bigg)$ & No & Yes & General Parameterization \\
	\hline
        This paper (Theorem \ref{thm_global_convergence_2}) & $\tilde{\mathcal{O}}\bigg(1/\sqrt{T}\bigg)$  & $\tilde{\mathcal{O}}\bigg(1/\sqrt{T}\bigg)$  & No & Yes & General Parameterization \\
	\hline
         This paper (Theorem \ref{thm_global_convergence_1}) & $\tilde{\mathcal{O}}\bigg(1/T^{0.5-\epsilon}\bigg)$  & $\tilde{\mathcal{O}}\bigg(1/T^{0.5-\epsilon}\bigg)$  & Yes \footnotemark  & Yes & General Parameterization \\
	\hline
        \multicolumn{6}{|c|}{\vspace{-0.1cm}} \\ 
        \hline
        \textbf{Lower bound} \citep{auer2008near} & $\Omega\bigg(1/\sqrt{T}\bigg)$  & N/A  & N/A & N/A & N/A\\
	\hline
	\end{tabular}}
        \caption{This table summarizes the different model-based and model-free state-of-the-art algorithms available in the literature for average reward CMDPs and their results for global convergence rate and average constraint violation.
        The bounds of global convergence describe convergence to the best performance permitted by the chosen features and policy class; the residual approximation floor is fixed (independent of $T$), and the listed rates govern how fast the statistical gap decays toward that floor.
 General parameterization refers to parameterizations whose policy score
$\nabla_{\theta}\log\pi_{\theta}(a|s)$ is uniformly bounded and Lipschitz in $\theta$, together with a Fisher non-degeneracy condition. }
	\label{table2}
\end{table*}


\footnotetext{The convergence result holds under the condition that $T\geq \tilde{\cO}\left(\tau_{\mathrm{mix}}^{2/\epsilon}\right)$, where $\tau_{\mathrm{mix}}$ is the mixing time.}
\textbf{Challenges and contributions:} In this paper, we propose a primal-dual-based natural actor-critic algorithm that achieves global convergence and constraint violation rates of $\Tilde{\mathcal{O}}(1/\sqrt{T})$ over a horizon of length $T$ if the mixing time, $\tau_{\mathrm{mix}}$, is known. On the other hand, in the absence of knowledge of $\tau_{\mathrm{mix}}$, the achievable rates change to $\tilde{\mathcal{O}}(1/T^{0.5-\epsilon})$, provided that the horizon length $T\geq \tilde{\cO}(\tau_{\mathrm{mix}}^{2/\epsilon})$. Therefore, even with unknown $\tau_{\mathrm{mix}}$, the achieved rates can be driven arbitrarily close to the optimal one by choosing small enough $\epsilon$, albeit at the cost of a large $T$. Both results significantly improve upon the state-of-the-art convergence rate of $\tilde{\mathcal{O}}(1/T^{1/5})$ in \cite{bai2024learning} for general policy parametrization and match the theoretical lower bound.

Since CMDP does not have strong convexity in the policy parameters with general parametrization, directly applying the primal-dual approach does not lead to optimal guarantees. The lack of strong convexity prevents the convergence result for the dual problem from automatically translating to the primal problem. This issue is reflected in the current state of art convergence rate in average reward CMDPs in \cite{bai2024learning}, which is $\tilde{\mathcal{O}}(1/{T}^{1/5})$ using a primal-dual approach, whereas the unconstrained counterpart in \cite{bai2023regret} on which the above is built upon achieves a convergence rate of $\tilde{\mathcal{O}}(1/{T}^{1/4})$. Thus, the previous literature for CMDP in general parametrized setups shows a gap in the results of the unconstrained and constrained setups. 

It is evident that directly adding a primal-dual structure to an existing actor-critic algorithm, such as in \cite{ganesh2024accelerated}, does not result in $\tilde{\mathcal{O}}(1/\sqrt{T})$ convergence rate for CMDP with general parametrization. This is specifically due to the existence of dual learning rate $\beta$, which eventually becomes one of the dominant terms in the convergence rate of the Lagrange function. If the dual learning rate is too low, the constraint violation converges very slowly. However, if $\beta$ is too high, the primal updates may exhibit high variance, slowing down convergence to the optimal policy. To navigate this tradeoff, it is crucial to carefully tune the problem parameters. 

Our algorithm has a nested loop structure where the outer loop of length $K$ updates the primal-dual parameters, and the inner loops of length $H$ run the natural policy gradient (NPG) and optimal critic parameter finding subroutines. We show that, to achieve the optimal rates, the parameter $H$ should be nearly constant while $K$ should be $\tilde{\Theta}(T)$. This adjustment makes our algorithm resemble a single-timescale algorithm. In contrast, for unconstrained MDPs, $K$ and $H$ are typically set to $\Theta(\sqrt{T})$ \cite{ganesh2024accelerated}. Achieving $\mathcal{O}(1/\sqrt{T})$ convergence in our setting requires ensuring that the bias remains sufficiently small despite a significantly low value of $H$. Achieving this challenging goal is made possible due to the use of Multi-level Monte Carlo (MLMC)-based estimates in the inner loop subroutines and a sharper analysis of these updates. One benefit of using MLMC-based estimates is that it allows us to design an algorithm that achieves near-optimal global convergence and constraint violation rates without knowledge of the mixing time. Related unconstrained average-reward results have also used MLMC to avoid mixing-time oracles \citep{suttle2023beyond, patel2024towards,ganesh2024accelerated}; here we apply this idea to CMDPs via a primal–dual natural actor–critic approach and provide global-rate and average-violation guarantees. We would like to emphasize that our work is the first in the general parameterized CMDP literature to achieve this feat.
\vspace{-.03in}


\section{Formulation}

\noindent Consider an infinite-horizon average reward constrained Markov Decision Process (CMDP) denoted as $\mathcal{M} = (\mathcal{S}, \mathcal{A}, r, c, P, \rho)$ where $\mathcal{S}$ is the state space, $\mathcal{A}$ is the action space of size $A$, $r: \mathcal{S} \times \mathcal{A} \rightarrow [0, 1]$ represents the reward function, $c: \mathcal{S} \times \mathcal{A} \rightarrow [-1, 1]$ denotes the constraint cost function, $P: \mathcal{S} \times \mathcal{A} \rightarrow \Delta({\mathcal{S}})$ is the state transition function where $\Delta(\mathcal{S})$ denotes the probability simplex over $\mathcal{S}$, and $\rho\in \Delta(\mathcal{S})$ indicates the initial distribution of states. A policy $\pi: \mathcal{S} \rightarrow \Delta(\mathcal{A})$ maps a state to an action distribution. The average reward and average constraint cost of a policy, $\pi$, denoted as $J_r^\pi$ and $J_c^\pi$ respectively, are defined as
\small
\begin{equation}
    J_{g}^\pi \triangleq \lim_{T \rightarrow \infty} \frac{1}{T} \mathbb{E}_{\pi} \left[ \sum_{t=0}^{T-1} g(s_t, a_t) \bigg| s_0 \sim \rho \right], g\in\{r,c\}
\end{equation}
\normalsize
The expectations computed are over the distribution of all $\pi$-induced trajectories $\{(s_t, a_t)\}_{t=0}^\infty$ where $a_t \sim \pi(s_t)$, and $s_{t+1} \sim P(s_t, a_t)$, $\forall t \in \{0, 1, \cdots \}$. For simplifying the notation, we drop the dependence on $\rho$ when there is no confusion. We aim to maximize the average reward while ensuring that the average cost exceeds a given threshold. Without loss of generality, we can formulate this as:
\begin{equation} \label{eq:unparametrized_formulation}
    \max_\pi J_{r}^\pi \text{ s.t. } J_c^\pi \geq 0
\end{equation}
When the underlying state space, $\mathcal{S}$ is large, the above problem becomes difficult to solve because the search space of the policy, $\pi$, increases exponentially with $|\mathcal{S}\times\mathcal{A}|$. We, therefore, consider a class of parametrized policies, $\{\pi_\theta | \theta \in \Theta\}$ that indexes the policies by a $d$-dimensional parameter, $\theta \in \mathbb{R}^d$ where $d \ll |\mathcal{S}||\mathcal{A}|$. The original problem in (2) can then be reformulated as follows.
\begin{equation}
    \label{eq:objective}
    \max_{\theta \in \Theta} {J}_r^{\pi_\theta} \text{ s.t. } J_{c}^{\pi_\theta} \geq 0
\end{equation}
In the remaining article, we use $J_g^{\pi_\theta} = J_g(\theta)$ for $g \in \{r, c\}$ for simplifying the notation. We assume the optimization problem in \eqref{eq:objective} obeys the Slater condition, which ensures the existence of an interior point solution. This assumption is commonly used in model-free average-reward CMDPs \cite{wei2022provably, bai2024learning}.
\begin{assumption}[Slater condition]
    \label{ass_slater}
    There exists a $\delta\in (0, 1)$ and $\bar{\theta} \in \Theta$ such that $J_{c}(\bar{\theta}) \geq \delta$.
\end{assumption}
We now make the following assumption on the CMDP.
\begin{assumption}
    \label{assump:ergodic_mdp}
    The CMDP $\mathcal{M}$ is ergodic, i.e., the Markov chain, $\{s_t\}_{t\geq0}$, induced under every policy $\pi$, is irreducible and aperiodic.
\end{assumption}

Note that most works on average reward MDPs and CMDPs with general parameterizations, to the best of our knowledge, assume ergodicity \cite{bai2024learning, ganesh2024variance,bai2023regret}. If CMDP $\mathcal{M}$ is ergodic, then $\forall \theta\in\Theta$, there exists a $\rho$-independent unique stationary distribution $d^{\pi_{\theta}}\in \Delta(\mathcal{S})$, that obeys $P^{\pi_{\theta}}d^{\pi_{\theta}}=d^{\pi_{\theta}}$ where $P^{\pi_\theta}(s, s')= \sum_{a\in\mathcal{A}}\pi_{\theta}(a|s)P(s'|s, a)$, $\forall s, s'\in\mathcal{S}$. Since $\forall \pi$, $J_g^{\pi} = \sum_{s, a} g(s, a)\pi(a|s) d^{\pi}(s)$, the assumption of ergodicity also ensures that $J_g^{\pi}$ is independent of $\rho$, $\forall g\in \{r, c\}$. Next, we define the mixing time of an ergodic CMDP.

\begin{definition}
The mixing time of a CMDP $\mathcal{M}$ with respect to a policy parameter $\theta$ is defined as, $\tau_{\mathrm{mix}}^{\theta}\coloneqq \min\left\lbrace t\geq 1\bigg| \|(P^{\pi_{\theta}})^t(s, \cdot) - d^{\pi_\theta}\|\leq \dfrac{1}{4}, \forall s\in\mathcal{S}\right\rbrace$. We also define $\tau_{\mathrm{mix}}\coloneqq \sup_{\theta\in\Theta} \tau^{\theta}_{\mathrm{mix}} $ as the overall mixing time. In this paper,  $\tau_{\mathrm{mix}}$ is finite due to ergodicity.
\end{definition}

The mixing time of an MDP measures how fast the MDP reaches its stationary distribution when executing a fixed policy. We use a primal-dual actor-critic approach to solve \eqref{eq:objective}. Before proceeding further, we define a few terms. The action-value function corresponding to a policy $\pi_\theta$ is given as 
\small
\begin{equation} \label{q_function}
    \begin{aligned}
    &Q_g^{\pi_{\theta}}(s,a) = \mathbb{E}_\pi\Bigg[\sum_{t=0}^{\infty} \bigg\{ g(s_t,a_t) - J_g(\theta)\bigg\}\bigg| s_0=s, a_0=a\Bigg]
    \end{aligned}
\end{equation}
\normalsize
where $g \in \{r,c\}$ and $(s, a)\in \mathcal{S}\times \mathcal{A}$. We further write their corresponding state value function as
\begin{equation}\label{v_function}
    \begin{aligned}
     V_g^{\pi_{\theta}}(s) = \mathbb E_{a \sim \pi_{\theta}(s)}[Q_g^{\pi_{\theta}}(s,a)]
    \end{aligned}
\end{equation}
The Bellman's equation can be expressed as follows \citep{puterman2014markov} $\forall (s, a)\in \mathcal{S}\times\mathcal{A}$, and $g\in\{r, c\}$.
\begin{equation}\label{bellman}
    \begin{aligned}
    Q_g^{\pi_{\theta}}(s, a)
    = g(s,a) - J_g(\theta) + \mathbb{E}_{s'\sim P(s, a)}\left[V_g^{\pi_{\theta}}(s')\right],
    \end{aligned}
\end{equation}

\begin{wrapfigure}{r}{0.5\textwidth}
  \vspace{-22pt}
  \begin{center}
  \fbox{%
    \begin{minipage}{0.5\textwidth}
    \vspace{-.21in}
    \begin{algorithm}[H]
    \caption{Primal-Dual Natural Actor-Critic (PDNAC)}
    \begin{algorithmic}[1]
        \State \textbf{Input:} Initial parameters $\theta_0$, $\omega_0=0$, $\xi_0=0$, $\lambda_0 = 0$, policy update stepsize $\alpha$, dual update stepsize $\beta$, NPG update stepsize $\gamma_{\omega}$, critic update stepsize $\gamma_{\xi}$, initial state $s_0 \sim \rho(\cdot)$, outer loop size $K$, inner loop size $H$,  $T_{\max}$
	  \For{$k = 0,\cdots, K-1$}
            \State $\omega_{g,0}^{k} \gets \omega_0$, ~ $\xi_{g,0}^{k} \gets \xi_0$ $\forall g \in \{r, c\}$;
            \State \textcolor{blue}{/* Critic Subroutine */}
            \For{$h = 0,\cdots,H-1$}
                \State $s_{kh}^0\gets s_0$,  $P_h^k \sim \text{Geom}(1/2)$ \State $l_{kh}\gets (2^{P_h^k}-1)\mathbf{1}(2^{P_h^k}\leq T_{\max})+1$
		          \For{$t = 0, \dots, l_{kh}-1 $}
		              \State Take action $a^t_{kh} \sim \pi_{\theta_k}(\cdot | s^t_{kh})$
		                \State Observe $s^{t+1}_{kh} \sim P(\cdot | s_{kh}^{t}, a^t_{kh})$
		              \State Observe $g(s^t_{kh}, a^t_{kh})$, $g \in \{r,c\} $
		          \EndFor
                \State  $s_0\gets s_{kh}^{l_{kh}}$
		      \State Update $\xi_{g,h}^k$ using \eqref{eq_critic_inner_loop_update} and \eqref{eq:MLMC_Critic}.
            \EndFor
            \State \textcolor{blue}{/* Actor Subroutine */}
            \For{$h = 0,\cdots,H-1$}
                \State $s_{kh}^0\gets s_0$, $Q_h^k  \sim \text{Geom}(1/2)$
                \State $l_{kh}\gets (2^{Q_h^k}-1)\mathbf{1}(2^{Q_h^k}\leq T_{\max})+1$
		      \For{$t = 0, \dots, l_{kh}-1 $}
		          \State Take action $a^t_{kh} \sim \pi_{\theta_k}(\cdot | s^t_{kh})$
		          \State Observe $s^{t+1}_{kh} \sim P(\cdot | s_{kh}^{t}, a^t_{kh})$
		          \State Observe $g(s^t_{kh}, a^t_{kh})$, $g \in \{r,c\} $
		      \EndFor
                \State $s_0\gets s_{kh}^{l_{kh}}$
		      \State Update $\omega_{g,h}^k$ using \eqref{eq_omeqa_h_update} and \eqref{eq_grad_f_estimate}
            \EndFor
            \State $\xi_{g}^k\gets \xi_{g,H}^k$, $\omega_g^k \gets \omega_{g,H}^k$, $ g\in \{r, c\}$
            \State $\omega_k\gets \omega_r^k + \lambda_{k} \omega_c^k$ 
           \State Update $(\theta_{k}, \lambda_{k})$ using \eqref{eq:update}
        \EndFor
    \end{algorithmic}
    \label{alg:pdranac}
\end{algorithm}
  \vspace{-.21in}
 \end{minipage}
   }
  \end{center}
  \vspace{-50pt}
\end{wrapfigure}

We define the advantage term for the reward and cost functions as follows $A_g^{\pi_{\theta}}(s, a) \triangleq Q_g^{\pi_{\theta}}(s, a) - V_g^{\pi_{\theta}}(s)$ where $g \in \{r, c\}$, $(s, a)\in \mathcal{S}\times \mathcal{A}$. With the above notations, we now present the commonly-used policy gradient theorem established by \cite{sutton1999policy} for $g\in\{r, c\}$.  
\begin{align}
\label{eq_policy_grad_theorem}
        \nabla_{\theta} J_g(\theta)\!=\!\mathbb{E}_{\substack{s\sim d^{\pi_\theta}\\ a\sim \pi_\theta(\cdot\mid s)}}\!\bigg[ \!A_g^{\pi_{\theta}}(s,a)\nabla_{\theta}\log\pi_{\theta}(a|s)\!\bigg]
\end{align}

The above term is useful in policy gradient-type algorithms where the learning direction of the parameter $\theta$ is given by $\{\nabla_{\theta} J_g(\theta)\}_{g\in\{r, c\}}$. In this paper, however, we will be interested in the Natural Policy Gradients $\{\omega^*_{g, \theta}\}_{g\in\{r, c\}}$ defined as follows.
\begin{align}\label{eq:optNPG}
        \omega^{*}_{g, \theta} \triangleq F(\theta)^{-1} \nabla_{\theta} J_g(\theta),
    \end{align}
where $F(\theta)\in \mathbb{R}^{d\times d}$ is the Fisher information matrix, which is formally defined as: $F(\theta) = \E_{(s, a)\sim \nu^{\pi_{\theta}}} \left[\nabla_{\theta}\log\pi_{\theta}(a|s)\otimes\nabla_{\theta}\log\pi_{\theta}(a|s) \right]$ where $\nu^{\pi_\theta}(s, a) \triangleq d^{\pi_\theta}(s)\pi_\theta(a|s)$, $\forall (s, a)$, and $\otimes$ defines the outer product. Note that NPG is similar to PG except modulated by the Fisher matrix, which accounts for the rate of change of policies with $\theta$. The NPG direction $\omega^*_{g, \theta}$ can also be expressed as a solution to the following strongly convex optimization problem. 
\begin{align} \label{eq:NPG}
    \min_{\omega\in \mathbb{R}^d} f_g(\theta, \omega) \coloneqq \frac{1}{2}\omega^\top  F(\theta) \omega - \omega^{\top}\nabla_{\theta} J_g(\theta)
\end{align}
The above formulation allows one to calculate the NPG in a gradient-based iterative procedure. In particular, the gradient of $f_g(\theta, \cdot)$ is obtained as $\nabla_{\omega} f_g(\theta, \omega) = F(\theta)\omega - \nabla_{\theta} J_g(\theta)$.


\section{Algorithm}

We solve $\eqref{eq:objective}$ via a primal-dual algorithm based on the following saddle point optimization.
\begin{align}
\label{eq:def_saddle_point_opt}
    \begin{split}
    \max_{\theta\in\Theta}\min_{\lambda\geq 0}&~ \mathcal{L}(\theta, \lambda) \triangleq J_r(\theta) +\lambda J_c(\theta)
    \end{split}
\end{align}
The term $\mathcal{L}(\cdot, \cdot)$ is defined as the Lagrange function, and $\lambda$ is the Lagrange multiplier. Our algorithm aims updating $(\theta, \lambda)$ by the policy gradient iteration $\forall k\in\{0,  \cdots, K-1\}$ as shown below from an arbitrary initial point $(\theta_0, \lambda_0=0)$. 

\begin{align}
    \theta_{k+1} = \theta_k +\alpha F(\theta_k)^{-1}\nabla_{\theta}\mathcal{L}(\theta_k, \lambda_k), ~~\lambda_{k+1} = \mathcal{P}_{[0, \frac{2}{\delta}]} \left[\lambda_k - \beta J_c(\theta_k)\right]
\end{align}
where $\alpha$ and $\beta$ denote primal and dual learning rates respectively, and $\delta$ indicates the Slater parameter introduced in Assumption \ref{ass_slater}. Moreover, for any $\Lambda\subset \mathbb{R}$,  $\mathcal{P}_{\Lambda}$ indicates the projection onto $\Lambda$. Since  $\nabla_{\theta}\mathcal{L}(\theta_k, \lambda_k)$, $F(\theta_k)$, and $J_c(\theta_k)$ are not exactly computable due to lack of knowledge of the exact transition kernel, $P$, and thus, that of the occupancy measure, $\nu^{\pi_{\theta_k}}$, in most RL scenarios, we use the following approximate updates.
\begin{align}\label{eq:update}
    \theta_{k+1} = \theta_k+\alpha \omega_k, \lambda_{k+1} = \mathcal{P}_{[0,\frac{2}{\delta}]}[\lambda_k -\beta \eta_{c}^k]
\end{align}
where $\omega_k$ is the estimate of the NPG $F(\theta_k)^{-1}\nabla_{\theta} \mathcal{L}(\theta_k, \lambda_k)$, while $\eta_{c}^k$ estimates $J_c(\theta_k)$. Below, we discuss the detailed procedure to compute these estimates. 

\subsection{Estimation Procedure}

We formally characterize our detailed algorithm in Algorithm \ref{alg:pdranac}, which utilizes a Multi-Level Monte Carlo (MLMC)-based Actor-Critic algorithm to compute the estimates stated above. For the exposition of our algorithm, it is beneficial to first discuss the critic estimation procedure, where the goal is to estimate the value functions and the average reward and cost functions. These estimates are then further used to obtain the NPG estimates, which are, in turn, used to update the policy parameter. The algorithm runs in $K$ epochs (also called \textit{outer loops}). At the start of $k$th epoch, the primal and dual parameters are denoted as $\theta_{k}, \lambda_{k}$ respectively.

\subsubsection{Critic Estimation}
\label{sec_critic_estimate}

At the $k$th epoch, one of the tasks in critic estimation is to obtain an estimate of $J_g(\theta_k)$, $g\in \{r, c\}$. Note that $J_g(\theta_k)$ can be written as a solution to the following optimization.
\begin{align}
	\min_{\eta\in \mathbb{R}} R_g(\theta_k, \eta) \coloneqq \frac{1}{2} \sum_{s\in \mathcal{S}} \sum_{a\in \mathcal{A}} \nu^{\pi_{\theta_k}}_g(s, a) \left\{ \eta - g(s, a)\right\}^2
\end{align}
The second task is to estimate the value function $V_g^{\pi_{\theta_k}}$. To facilitate this objective, it is assumed that, $\forall g\in \{r, c\}$, the value function $V_g^{\pi_{\theta_k}}(\cdot)$ is well-approximated by a linear critic function $\hat{V}_g(\zeta_g^{\theta_k}, \cdot)\coloneqq \langle \phi_g(\cdot), \zeta_{g}^{\theta_k}\rangle$ where $\phi_g:\mathcal{S}\rightarrow \mathbb{R}^m$ is a feature mapping with the property that $\|\phi_g(s)\|\leq 1$, $\forall s\in \mathcal{S}$, and $\zeta^{\theta_k}_{g}\in \mathbb{R}^m$ is a solution of the following optimization.
\begin{align}\label{value_function_problem}
	\min_{\zeta\in\mathbb{R}^m}  E_g(\theta_k, \zeta) \coloneqq \frac{1}{2}\sum_{s \in \mathcal{S}} d^{\pi_{\theta_k}}(s) (V_g^{\pi_{\theta_k}}(s) - \hat{V}_g(\zeta, s))^2
\end{align} 
Note that the gradients of the functions $R_g(\theta_k, \cdot)$, $E_g(\theta_k, \cdot)$ can be obtained as follows.
\begin{align}
    \label{eq_15_washim}
    \nabla_{\eta} R_g(\theta_k, \eta) &= \sum_{s \in \mathcal{S}} \sum_{a\in \mathcal{A}} \nu_{g}^{\pi_{\theta_k}}(s, a) \left\{ \eta - g(s, a) \right\}, ~\forall \eta\in \mathbb{R},\\
    \nabla_{\zeta} E_g(\theta_k, \zeta) &= \sum_{s \in \mathcal{S}} \sum_{a\in \mathcal{A}} \nu_g^{\pi_{\theta_k}}(s, a) \left( \zeta^{\top} \phi_g(s) - Q_g^{\pi_{\theta_k}}(s, a)\right)\phi_g(s), ~\forall \zeta\in \mathbb{R}^{m} \nonumber\\
    \label{eq_16_washim}
    &\hspace{-1.5cm}\overset{(a)}{=}\sum_{s \in \mathcal{S}} \sum_{a\in \mathcal{A}} \nu_g^{\pi_{\theta_k}}(s, a) \left( \zeta^{\top} \phi_g(s) - g(s, a) + J_g(\theta_{k}) - \mathbb{E}_{s'\sim P(s, a)}[V_g^{\pi_{\theta_{k}}}(s')]\right)\phi_g(s)
\end{align}
where $(a)$ results from Bellman's equation \eqref{bellman}. Since the above gradients cannot be exactly obtained due to lack of knowledge of the transition model $P$, and thus that of $\nu_{g}^{\pi_{\theta_{k}}}$, we execute the following $H$ \textit{inner loop} steps to compute approximations of $J_g(\theta_{k})$ and $\zeta_{g}^{\theta_{k}}$ from $\eta_{g, 0}^k =0$ and $\zeta_{g, 0}^k = 0$.
\begin{align}
    \eta_{g,h+1}^k = \eta_{g,h}^k- c_{\gamma} \gamma_{\xi} \hat{\nabla}_{\eta} R_g(\theta_{k}, \eta_{g,h}^k), ~~\zeta_{g,h+1}^k = \zeta_{g,h}^k-\gamma_{\xi}\hat{\nabla}_{\zeta} E_g(\theta_{k}, \xi_{g, h}^k) 
    \label{eq:critic-split}
\end{align}
where $c_{\gamma}$ is a constant, $\gamma_{\xi}$ defines a learning parameter, and $h\in\{0, \cdots, H-1\}$. Moreover, the terms $\hat{\nabla}_{\eta} R_g(\theta_{k}, \eta_{g,h}^k)$ and $\hat{\nabla}_{\zeta} E_g(\theta_{k}, \xi_{g, h}^k)$ indicate estimates of $\nabla_{\eta} R_g(\theta_{k}, \eta_{g,h}^k)$ and $\nabla_{\zeta} E_g(\theta_{k}, \zeta_{g, h}^k)$ respectively where $\xi_{g, h}^k \triangleq [\eta_{g,h}^k, (\zeta_{g, h}^k)^{\top}]^{\top}$. Because the expression of $\nabla_{\zeta} E_g(\theta_{k}, \zeta_{g, h}^k)$ comprises $\zeta_{g, h}^k$ and $J_g(\theta_{k})$ (refer \eqref{eq_16_washim}), its estimate is a function of $\eta_{g, h}^k$ and $\zeta_{g, h}^k$. At the end of inner loop iterations, we obtain $\eta^k_{g, H}$ and $\zeta_{g, H}^k$ which are estimates of $J_g(\theta_{k})$ and $\zeta_{g}^{\theta_{k}}$ respectively. It remains to see how the gradient estimates used in \eqref{eq:critic-split} can be obtained. Note that \eqref{eq:critic-split} can be compactly written as
\begin{align}
    \label{eq_critic_inner_loop_update}
    \xi_{g,h+1}^k &= \xi_{g,h}^k - \gamma_{\xi} \mathbf{v}_g(\theta_{k}, \xi_{g,h}^k),~~\mathbf{v}_g(\theta_{k}, \xi_{g,h}^k) \triangleq [c_{\gamma}\hat{\nabla}_{\eta} R_g(\theta_{k}, \eta_{g,h}^k), (\hat{\nabla}_{\zeta} E_g(\theta_{k}, \xi_{g, h}^k))^{\top}]^{\top}
\end{align}
For a given pair $(k,h)$, let $\mathcal{T}_{kh} = \{(s_{kh}^t, a_{kh}^t, s_{kh}^{t+1})\}_{t=0}^{l_{kh}-1}$ indicate a $\pi_{\theta_{k}}$-induced trajectory of length $l_{kh}=2^{Q_h^k}$ where $Q_{h}^k\sim \mathrm{Geom}(1/2)$. Following \eqref{eq_15_washim} and \eqref{eq_16_washim}, an estimate of $\mathbf{v}_g(\theta_k, \xi_{g,h}^k)$ based on a single state transition sample $z_{kh}^j = (s_{kh}^j, a_{kh}^j, s_{kh}^{j+1})$ can be obtained as
\begin{align}
    \label{eq_19_washim}
    \mathbf{v}_g&(\theta_k,\xi_{g,h}^k;z_{kh}^j) \nonumber\\
    &= \underbrace{
    \begin{pmatrix}
	c_{\gamma} & 0 \\
        \phi_g(s_{kh}^{j}) & \phi_g(s_{kh}^{j}(\phi_g(s_{kh}^{j})-\phi_g(s_{kh}^{j+1}))^\top  
    \end{pmatrix}}_{\triangleq~ \mathbf{A}_g(\theta_{k}; z_{kh}^j)} \xi_{g,h}^k - \underbrace{\begin{pmatrix}
		c_{\gamma} g(s_{kh}^{j},a_{kh}^{j}) \\
		g(s_{kh}^j,a_{kh}^j)\phi_g(s_{kh}^{j})
		\end{pmatrix}}_{\triangleq~ \mathbf{b}_g(\theta_{k}; z_{kh}^j)} 
\end{align}
Applying MLMC-based estimation, the term $\mathbf{v}_g(\theta_{k}, \xi_{g, h}^{k})$ is finally computed as
\begin{align}
    \label{eq:MLMC_Critic}
    \mathbf{v}_{g}(\theta_k, \xi_{g, h}^k) &= \mathbf{v}_{g,kh}^0 +
    \begin{cases}
        2^{Q_h^k}(\mathbf{v}_{g,kh}^{Q_h^k} - \mathbf{v}_{g,kh}^{Q_h^k - 1}),& \text{if } 2^{Q_h^k}\leq T_{\max}\\
	0, & \text{otherwise}
    \end{cases}    
\end{align}
where $T_{\max}$ is a constant, $\mathbf{v}^j_{g, kh} = 2^{-j}\sum_{t=0}^{2^j-1} \mathbf{v}_g(\theta_k, \xi_{g, h}^k; z_{kh}^t)$, $j\in \{0,Q_h^k-1, Q_h^k\}$.
Note that the maximum number of state transition samples utilized in the MLMC estimate is $T_{\max}$. Moreover, it can be demonstrated that the samples used \textit{on average} is $\tilde{O}(\log T_{\max})$. The advantage of the MLMC estimator is that it achieves the same bias as averaging \(T_{\max}\) samples, but requires only \(\Tilde{\cO} (\log T_{\max})\) samples. In addition, since drawing from a geometric distribution does not require knowledge of the mixing time, we can eliminate the mixing time knowledge assumption used in previous works \cite{bai2023regret, ganesh2024variance, bai2024learning}. Furthermore, these previous works utilize policy gradient methods and require saving the trajectories of length $H$ for gradient estimations, while our approach does not. Therefore, our algorithm reduces the memory complexity by a factor of $H$, which is significant since the choice of $H$ in these works scales with mixing and hitting times of the MDP.

\subsubsection{Natural Policy Gradient (NPG) Estimator}
\label{sec_npg_estimate}

Recall that the outcome of the critic estimation at the $k$th epoch is $\xi_{g,H}^k = [\eta_{g,H}^k, (\zeta_{g, H}^k)^{\top}]^{\top}$ where $\eta_{g,H}^k$, $\zeta_{g, H}^k$ estimate $J_g(\theta_{k})$ and the critic parameter $\zeta_g^{\theta_{k}}$ respectively. For simplicity, we will denote $\xi_{g,H}^k$ as $\xi_{g}^k = [\eta_{g}^k, (\zeta_{g}^k)^{\top}]^{\top}$. We estimate the NPG $\omega_{g, \theta_k}^*$ (refer to the definition \eqref{eq:optNPG}) using a $H$ step inner loop as stated below $\forall h\in \{0, \cdots, H-1\}$ starting from $\omega_{g, 0}^k=0$.
\begin{align}
    \label{eq_omeqa_h_update}
    \omega_{g,h+1}^k = \omega_{g,h}^k - \gamma_{\omega} \hat{\nabla}_{\omega} f_g(\theta_{k}, \omega_{g, h}^k, \xi_{g}^k)
\end{align}
where $\hat{\nabla}_{\omega} f_g(\theta_{k}, \omega_{g, h}^k, \xi_{g}^k)$ is an MLMC-based estimate of $\nabla_{\omega} f_g(\theta_{k}, \omega_{g, h}^k)$ where $f_g$ is given in \eqref{eq:NPG}. To obtain this estimate, a $\pi_{\theta_k}$-induced trajectory $\mathcal{T}_{kh} = \{(s_{kh}^t, a_{kh}^t, s_{kh}^{t+1})\}_{t=0}^{l_{kh}-1}$ of length $l_{kh} = 2^{P_h^k}$ is considered where $P_h^k\sim \mathrm{Geom}(1/2)$. For a certain transition $z_{kh}^j = (s_{kh}^j, a_{kh}^j, s_{kh}^{j+1})$, define the following estimate.
\begin{align}
    \label{eq_22_washim_new}
    &\hat{\nabla}_{\omega} f_g(\theta_{k}, \omega_{g,h}^k, \xi_{g}^k ; z_{kh}^j) = \hat{F}(\theta_{k}; z_{kh}^j)\omega_{g, h}^k - \hat{\nabla}_{\theta} J_g(\theta_{k}, \xi_{g}^k; z_{kh}^j)~\text{where}~\\
    &\hat{F}(\theta_{k}; z_{kh}^j) = \nabla_{\theta} \log \pi_{\theta_{k}}(a_{kh}^j|s_{kh}^j) \otimes \nabla_{\theta} \log \pi_{\theta_{k}}(a_{kh}^j|s_{kh}^j) \nonumber \\
    & \hat{\nabla}_{\theta} J_g(\theta_{k}, \xi_{g}^k; z_{kh}^j) = \hat{A}_g^{\pi_{\theta_{k}}}(\xi_{g}^k; z_{kh}^j) \nabla_{\theta} \log \pi_{\theta_{k}}(a_{kh}^j|s_{kh}^j)\nonumber\\
    & \hspace{2.3cm} \overset{(a)}{=}\left[ g(s_{kh}^j, a_{kh}^j) - \eta_g^k + \zeta_{g}^k\left(\phi_g(s_{kh}^{j+1}) - \phi_g(s_{kh}^j)\right) \right] \nabla_{\theta} \log \pi_{\theta_{k}}(a_{kh}^j|s_{kh}^j)\nonumber
\end{align}
where the advantage estimate used in $(a)$ is essentially a temporal difference (TD) error. Notice that the estimate of the policy gradient $\nabla_{\theta} J_g(\theta_k)$ depends on $\xi_g^k$ obtained in the critic estimation process. The MLMC-based estimate, therefore, can be obtained as follows.
\begin{align}
    \label{eq_grad_f_estimate}
    \hat{\nabla}_{\omega} f_g(\theta_{k}, \omega_{g, h}^k, \xi_{g}^k)  &= \mathbf{u}_{g,kh}^0 +
    \begin{cases}
        2^{P_h^k}(\mathbf{u}_{g,kh}^{P_h^k} - \mathbf{u}_{g,kh}^{P_h^k - 1}),& \text{if } 2^{P_h^k}\leq T_{\max}\\
	0, & \text{otherwise}
    \end{cases}    
\end{align}
where $\mathbf{u}_{g, kh}^j = 2^{-j}\sum_{t=0}^{2^j-1} \hat{\nabla}_{\omega} f_g(\theta_{k}, \omega_{g, h}^k, \xi_{g}^k ; z_{kh}^t)$, $j\in \{0, P_h^k-1, P_h^k\}$. The above estimate is applied in \eqref{eq_omeqa_h_update}, which finally yields the NPG estimate $\omega_{g, H}^k$. For simplicity, we denote $\omega_{g, H}^k$ as $\omega_g^k$.

\subsection{Primal and Dual Updates}

The estimates $\omega_{g}^k$, $g\in\{r, c\}$ obtained in section  \ref{sec_npg_estimate} can be combined to compute $\omega_k = \omega_{r}^k+\lambda_k \omega_c^k$. Moreover, section \ref{sec_critic_estimate} provides $\eta_{c}^k$ which is an estimate of $J_c(\theta_{k})$. At the $k$th epoch, these estimates can be used to update the policy parameter $\theta_{k}$ and the dual parameter $\lambda_{k}$ following \eqref{eq:update}.

\section{Global Convergence Analysis}
\label{sec_global_conv_ananalysis}

We first state some assumptions that we will be using before proceeding to the main results. Define $\mathbf{A}_{g}(\theta) = \E_{\theta} \mathbf{A}_g(\theta;z)$ where the expectation is over the distribution of $z=(s, a, s')$ where $(s, a)\sim \nu_{g}^{\pi_{\theta}}$, $s'\sim P(s, a)$, and the term $\mathbf{A}_g(\theta; z)$ is defined in \eqref{eq_19_washim}. Similarly, $\mathbf{b}_g(\theta) \triangleq \E_{\theta}\left[\mathbf{b}_g(\theta; z)\right]$. Let $\xi_{g}^*(\theta) = [\mathbf{A}_g(\theta)]^{-1}\mathbf{b}_g(\theta)=[\eta_{g}^*(\theta), \zeta_{g}^*(\theta)]$. With these notations, we are ready to state the following assumptions regarding critic approximations.

\begin{assumption}
    \label{assump:critic-error}
    For $g \in \{r, c\}$, we define the worst-case critic approximation error to be $\epsilon_g^{\mathrm{app}} = \sup_{\theta}  \E_{s \sim d^{\pi_\theta}} \left[ (\zeta_{g}^*(\theta))^{\top}\phi_g(s) - V_g^{\pi_{\theta}}(s) \right]^2$. We assume  $\epsilon^{\mathrm{app}} \triangleq \max \{ \epsilon_r^{\mathrm{app}} , \epsilon_c^{\mathrm{app}}\}$ to be finite.
\end{assumption}

\begin{assumption} 
    \label{assum:critic_positive_definite}
    There exists  $\lambda > 0$ such that $\E_{\theta}[\phi_g(s)\left(\phi_g(s) - \phi_g(s')\right)]-\lambda I$ is positive semi-definite $\forall\theta \in \Theta$ and $g \in \{r, c\}$.
\end{assumption}

Both Assumptions \ref{assump:critic-error} and \ref{assum:critic_positive_definite} are frequently used in analyzing actor-critic methods \cite{bhandari2018, zou2019finite, qiu2021finite, xu2020improved, suttle2023beyond}. Assumption \ref{assump:critic-error} intuitively relates to the quality of the feature mapping where $\epsilon_{\mathrm{app}}$ serves as a measure of this quality: well-crafted features result in small $\epsilon_{\mathrm{app}}$, whereas poorly designed features lead to a larger worst-case error. On the other hand, Assumption \ref{assum:critic_positive_definite}, is essential for ensuring the convergence of the critic updates. It can be shown that Assumption \ref{assum:critic_positive_definite} ensures that the matrix $\mathbf{A}_g(\theta)$ is invertible for sufficiently large $c_{\gamma}$ (details in the appendix).
 
\begin{assumption}
    \label{assump:function_approx_error}
    Define the following function for $\theta, \omega\in\Rd$, $\lambda\geq 0$, and $\nu\in \Delta(\mathcal{S}\times \mathcal{A})$.
    \begin{align}
        \label{equ: minimal compatible function approximation error}
        \begin{split}
            &L_{\nu}(\omega,\theta,\lambda) =\E_{(s, a)\sim \nu}\bigg[\bigg(\nabla_\theta\log\pi_{\theta}(a\vert s)\cdot\omega -A_{r}^{\pi_\theta}(s,a) - \lambda A_c^{\pi_\theta}(s,a) \bigg)^2\bigg]
        \end{split}
    \end{align} 
    Let $\omega_{\theta, \lambda}^* = \arg\min_{\omega\in \Rd} L_{\nu^{\pi_{\theta}}}(\omega, \theta, \lambda)$. It is assumed that $L_{\nu^{\pi^*}}(\omega_{\theta, \lambda}^*, \theta, \lambda)\leq \epsilon_{\mathrm{bias}}$ for $\theta\in \Theta$ and $\lambda\in [0, 2/\delta]$ where $\pi^*$ is the solution to the optimization \eqref{eq:unparametrized_formulation}, and $\epsilon_{\mathrm{bias}}$ is a positive constant. The LHS of the above inequality is called the \textit{transferred compatible function approximation error}. Note that it can be easily verified that $\omega^*_{\theta, \lambda}$ is the NPG update direction $\omega^*_{\theta, \lambda} = F(\theta)^{-1}\nabla_\theta \mathcal{L}(\theta, \lambda)$.
\end{assumption}

\begin{assumption}
    \label{assump:score_func_bounds}
    For all $\theta, \theta_1,\theta_2 \in\Theta$ and $(s,a)\in\mathcal{S}\times\mathcal{A}$, the following holds for some $G_1, G_2>0$.
    \begin{align*}
        \begin{split}
            &(a) \Vert\nabla_{\theta}\log\pi_\theta(a\vert s)\Vert\leq G_1 \quad \\
            &(b) \Vert \nabla_{\theta}\log\pi_{\theta_1}(a\vert s)-\nabla_\theta\log\pi_{\theta_2}(a\vert s)\Vert\leq G_2\Vert \theta_1-\theta_2\Vert.  
        \end{split}
    \end{align*}
\end{assumption}

\begin{assumption}[Fisher non-degenerate policy]
    \label{assump:FND_policy}
    There exists a constant $\mu>0$ such that $F(\theta)-\mu I_{d}$ is positive semidefinite, where $I_{d}$ denotes an identity matrix of dimension $d$.
\end{assumption}

{\bf Comments on Assumptions \ref{assump:function_approx_error}-\ref{assump:FND_policy}:}  We emphasize that these assumptions are commonly applied in the policy gradient (PG) literature \cite{liu2020improved,agarwal2021theory, papini2018stochastic, xu2019sample,fatkhullin2023stochastic}. The term $\epsilon_{\mathrm{bias}}$ reflects the parametrization capacity of $\pi_{\theta}$. For $\pi_{\theta}$ using the softmax parametrization, we directly have $\epsilon_{\mathrm{bias}}=0$ \cite{agarwal2021theory}. However, when $\pi_{\theta}$ employs a restricted parametrization that does not encompass all stochastic policies, $\epsilon_{\mathrm{bias}}$ is greater than zero. It is known that $\epsilon_{\mathrm{bias}}$ remains very small when utilizing rich neural parametrizations \cite{wang2019neural}. Assumption \ref{assump:score_func_bounds} requires that the score function is both bounded and Lipschitz continuous, which is a condition frequently assumed in the analysis of PG-based methods \cite{liu2020improved,agarwal2021theory, papini2018stochastic, xu2019sample,fatkhullin2023stochastic}. Assumption \ref{assump:FND_policy} ensures that the function $f_g(\theta, \cdot)$ is $\mu$-strongly convex by mandating that the eigenvalues of the Fisher information matrix are bounded from below. This is also a standard assumption for deriving global complexity bounds for PG based algorithms \cite{liu2020improved,zhang2021on,Bai_Bedi_Agarwal_Koppel_Aggarwal_2022,fatkhullin2023stochastic}. Recent studies have shown that Assumptions \ref{assump:score_func_bounds}-\ref{assump:FND_policy} hold true in various examples, including Gaussian policies with linearly parameterized means and certain neural parametrizations \cite{liu2020improved, fatkhullin2023stochastic}.

Before proving results for the convergence rate, we first provide the following lemma to associate the global convergence rate of the Lagrange function to the convergence of the actor and critic parameters. Similar ideas have been explored in \cite{ganesh2024accelerated,bai2024learning}. 
\begin{lemma}
    \label{lemma:local_global}
    If the policy parameters, $\{(\theta_k, \lambda_k)\}_{k=1}^K$ are updated via \eqref{eq:update}  and assumptions \ref{assump:function_approx_error}-\ref{assump:FND_policy} hold, then the following inequality is satisfied
    \begin{equation}
        \label{eq:general_bound}
	\begin{split}
            \frac{1}{K}&\E\sum_{k=0}^{K-1}\bigg(\mathcal{L}(\pi^*, \lambda_k)-\mathcal{L}(\theta_k,\lambda_k)\bigg)\leq \sqrt{\epsilon_{\mathrm{bias}}} +\frac{G_1}{K}\sum_{k=0}^{K-1}\E\Vert(\E_k\left[\omega_k\right]-\omega^*_k)\Vert \\
            &+\frac{\alpha G_2}{2K}\sum_{k=0}^{K-1}\E\Vert \omega_k\Vert^2+\frac{1}{\alpha K}\E_{s\sim d^{\pi^*}}[KL(\pi^*(\cdot\vert s)\Vert\pi_{\theta_0}(\cdot\vert s))],
        \end{split}
    \end{equation}
    where $KL(\cdot \|\cdot)$ is the Kullback-Leibler divergence, $\pi^*$ is the optimal policy for \eqref{eq:unparametrized_formulation} and $\omega^*_k \coloneqq \omega^*_{\theta_k,\lambda_k}$ is the exact NPG direction $F(\theta_k)^{-1}\nabla_{\theta}\mathcal{L}(\theta_k, \lambda_k)$. Finally, $\E_k$ denotes conditional expectation given history up to the $k$th iteration.
\end{lemma}
Observe the presence of $\epsilon_{\mathrm{bias}}$ in \eqref{eq:general_bound}. It dictates that due to the incompleteness of the policy class, the global convergence bound cannot be driven to zero. Note that the last term in \eqref{eq:general_bound} is $\cO(1/(\alpha K))$ because the term $\E_{s\sim d^{\pi^*}}[KL(\pi^*(\cdot\vert s)\Vert\pi_{\theta_0}(\cdot\vert s))]$ is a constant. The term related to $\E\|\omega_k\|^2$ can be further decomposed as follows.
\begin{align}
\label{eq:omegadeconposition}
    \begin{split}
        \dfrac{\alpha}{K}\sum_{k=0}^{K-1}\E\Vert \omega_k \Vert^2
        &\leq \dfrac{\alpha}{K}\sum_{k=0}^{K-1}\E\Vert \omega_k -\omega_k^*\Vert^2 + \dfrac{\alpha}{K}\sum_{k=0}^{K-1}\E\Vert \omega_k^* \Vert^2\\
        &\overset{(a)}{\leq} \dfrac{\alpha}{K}\sum_{k=0}^{K-1}\E\Vert \omega_k -\omega_k^*\Vert^2 + \dfrac{\alpha\mu^{-2}}{K}\sum_{k=0}^{K-1}\E\Vert \nabla_{\theta} \mathcal{L}(\theta_k, \lambda_k) \Vert^2
    \end{split}
\end{align}
where $(a)$ follows from Assumption \ref{assump:FND_policy} and the definition that $\omega_k^*=F(\theta_k)^{-1}\nabla_\theta \mathcal{L}(\theta_k,\lambda_k)$. We can obtain a global convergence bound by bounding the terms $\E\|\omega_k-\omega^*_k\|^2$, $\E\Vert(\E\left[\omega_k|\theta_k\right]-\omega^*_k)\Vert$ and $\E\Vert \nabla_{\theta}\mathcal{L}(\theta_k,\lambda_k)  \Vert^2$. The first two terms are the variance and bias of the NPG estimator, $\omega_k$, and the third term indicates the local convergence rate. Further, $\E\Vert \nabla_{\theta} \mathcal{L}(\theta_k,\lambda_k) \Vert^2$ can be upper bounded by a constant (Lemma \ref{lemma_grad_JL_bound} in the appendix). We now provide the convergence result for the actor and critic parameters. For brevity, we use $\lesssim$ to denote $\leq \tilde{\cO}(\cdot)$.
\begin{theorem}
    \label{thm:npg-main}
    Consider Algorithm \ref{alg:pdranac} and let Assumptions \ref{assump:ergodic_mdp}-\ref{assump:FND_policy} hold. If $J_g$ is $L$-smooth, $g \in \{r, c\}$, $\gamma_{\omega} = \frac{2 \log T}{\mu H}$ is such that $\gamma_\omega \leq \frac{\mu}{4(6 G_1^4 \tau_{\mathrm{mix}} \log T_{\max} + 2 G_1^2  \tau_{\mathrm{mix}}^2 \log T_{\max})}$, and $T_{\max}$ obeys $T_{\max} \geq \frac{8G_1^4 \tau_{\mathrm{mix}}}{\mu}$ where $\mu$ is defined in Assumption \ref{assump:FND_policy}, the following inequalities hold $\forall k\in \{0, \cdots, K-1\}$.
    \begin{align} 
        \label{eq:actorbound}
        \left\|\E_k[\omega_{g}^k] - \omega^{*}_{g,k}\right\|^2 &\lesssim \epsilon_{\mathrm{app}}+\dfrac{\tau_{\mathrm{mix}}^2}{T^2} \nonumber \\
        &+ \dfrac{\tau^2_{\mathrm{mix}}}{T_{\max}} \bigg\{\E_k\left[\norm{\xi_{g}^k-\xi_g^{*}(\theta_k)}^2\right] + \tau_{\mathrm{mix}}^2 \bigg\}+ \norm{\E_k[\xi_{g}^k]-\xi_g^{*}(\theta_k)}^2 ,\\
        \label{eq:actorbound2}
        \E_k\left[\|\omega_{g}^k -\omega^{*}_{g,k}\|^2\right] &\lesssim   \epsilon_{\mathrm{app}}+\left(\dfrac{1}{T^2}+\frac{\tau_{\mathrm{mix}}}{H} + \frac{\tau_{\mathrm{mix}}}{T_{\max}}\right)\tau_{\mathrm{mix}}^2+\E_k\left[\norm{\xi_{g}^k-\xi_g^{*}(\theta_k)}^2\right]
    \end{align}
    where $\omega^{*}_{g,k}$ is the NPG direction $F(\theta_k)^{-1}\nabla_{\theta}J_g(\theta_k)$, $\xi_g^*(\theta_k)$ is defined in section \ref{sec_global_conv_ananalysis}, and $\E_k$ denotes conditional expectation given history up to the $k$th iteration.
\end{theorem}

Theorem \ref{thm:npg-main} bounds the NPG bias $\|\E_k[\omega_g^k]-\omega_{g,k}^*\|$ and the second-order error $\E_k[\|\omega_{g}^k -\omega^{*}_{g,k}\|^2]$ in terms of the critic approximation error $\epsilon_{\mathrm{app}}$, and the bias $\|\E_k[\xi_g^k]-\xi_g^*(\theta_k)\|$ and the second-order error $\E_k [\|\xi_g^k - \xi_g^*(\theta_k)\|^2]$ in the critic parameter estimation. The following theorem provides bounds on these latter quantities.

\begin{theorem}
    \label{th:acc_td}
    Consider Algorithm \ref{alg:pdranac} and let Assumptions \ref{assump:ergodic_mdp}-\ref{assump:FND_policy} hold and $g \in \{r, c\}$. If we choose $\gamma_{\xi} = \frac{2\log T}{\lambda H}$ such that $\gamma_\xi \leq \frac{\lambda}{24 c_\gamma^2\tau_{\mathrm{mix}}  \log T_{\max}}$ while $T_{\max}$ obeys $T_{\max} \geq \frac{8c_\gamma^2 \tau_{\mathrm{mix}}}{\lambda}$ where $\lambda$ is defined in Assumption \ref{assum:critic_positive_definite}, the following inequalities hold $\forall k\in \{0, \cdots, K-1\}$.
    \begin{align} 
        \label{eq:criticbound}
        \left\|\E_k[\xi_{g}^k] - \xi^{*}_{g}(\theta_k)\right\|^2
        &\lesssim \dfrac{1}{T^2} + \dfrac{\tau^2_{\mathrm{mix}}}{T_{\max}}, \\    \label{eq:criticbound2} \E_k[\|\xi_{g}^k - \xi^{*}_g(\theta_k)\|^2] &\lesssim \dfrac{1}{T^2} + \frac{\tau_{\mathrm{mix}}}{H}+\frac{\tau_{\mathrm{mix}}}{T_{\max}} 
    \end{align}
    where $\xi^{*}_g(\theta_k)$ is defined in section \ref{sec_global_conv_ananalysis} and $\E_k$ denotes conditional expectation given history up to the $k$th iteration.
\end{theorem}

Invoking the bounds provided by Theorem \ref{thm:npg-main} and \ref{th:acc_td} into Lemma \ref{lemma:local_global}, we can obtain a convergence rate of the Lagrange function. Our next goal is to segregate the objective convergence and constraint violation rates from this Lagrange error. This is achieved by the following theorems. Depending on whether we have access to the mixing time, two similar but slightly different results can be obtained.

\begin{theorem}
    \label{thm_global_convergence_2}
    Consider the same setup and parameters as in Theorem \ref{th:acc_td} and Theorem \ref{thm:npg-main} and set $\alpha=T^{-1/2}$, $\beta =T^{-1/2}$. If $\tau_{\mathrm{mix}}$ is known, set $H=\tilde{\Theta}(\tau_{\mathrm{mix}}^2) $ with $K =T/H$. We have:
    \begin{align}
        & \frac{1}{K}\sum_{k=0}^{K-1}\E[J_r^{\pi^*}-J_r(\theta_k)]\lesssim \sqrt{\epsilon_{\mathrm{bias}}}  + \sqrt{\epsilon_{\mathrm{app}}} + \frac{1}{\sqrt{T}},\\
        &\frac{1}{K}\sum_{k=0}^{K-1}\E[-J_c(\theta_k)] \lesssim \sqrt{\epsilon_{\mathrm{bias}}}  + \sqrt{\epsilon_{\mathrm{app}}} + \frac{1}{\sqrt{T}}
    \end{align}
\end{theorem}

\begin{theorem}
    \label{thm_global_convergence_1}
    Consider the same setup and parameters as in Theorem \ref{th:acc_td} and Theorem \ref{thm:npg-main} and set $\alpha=T^{-1/2}$, $\beta =T^{-1/2}$, $H=T^{\epsilon}$ and $K =T^{1-\epsilon}$. With $T^{\epsilon} \geq \tilde{\Theta}(\tau_{\mathrm{mix}}^2)$,  we have:
    \begin{align}
        & \frac{1}{K}\sum_{k=0}^{K-1}\E[J_r^{\pi^*}-J_r(\theta_k)]\lesssim  \sqrt{\epsilon_{\mathrm{bias}}}  + \sqrt{\epsilon_{\mathrm{app}}} + \frac{1}{T^{(0.5-\epsilon)}},\\
        &\frac{1}{K}\sum_{k=0}^{K-1}\E[-J_c(\theta_k)] \lesssim     \sqrt{\epsilon_{\mathrm{bias}}}  + \sqrt{\epsilon_{\mathrm{app}}} + \frac{1}{T^{(0.5-\epsilon)}}
    \end{align}
\end{theorem}

Theorem \ref{thm_global_convergence_2} dictates that one can achieve both the objective convergence and the constraint violation rates as $\tilde{\cO}(T^{-1/2})$ up to some additive factors of $\epsilon_{\mathrm{bias}}$ and $\epsilon_{\mathrm{app}}$ where $T$ is the length of the horizon. However, knowledge of the mixing time, $\tau_{\mathrm{mix}}$, is needed in this case to set some parameters. On the other hand, Theorem \ref{thm_global_convergence_1} states that, if knowledge of $\tau_{\mathrm{mix}}$ is unavailable, one can achieve objective convergence and constraint violation rates as $\tilde{\cO}(T^{-1/2+\epsilon})$ as long as the horizon $T$ exceeds $\tilde{\Theta}(\tau_{\mathrm{mix}}^{2/\epsilon})$. Note that one can get arbitrarily close to the optimal rate of $\tilde{\cO}(T^{-1/2})$ by choosing arbitrarily small $\epsilon$. The caveat is that the smaller the $\epsilon$, the larger the horizon length needed to reach the desired rates.

In the setting of Theorem \ref{thm_global_convergence_1}, for small $\epsilon$, the horizon requirement becomes $T\gtrsim\tilde{O}\left(\tau_{\mathrm{mix}}^{2/\varepsilon}\right)$, which can be large for slowly mixing problems. In such a scenario, one can switch to the parameter setting of Theorem \ref{thm_global_convergence_2}, which does not impose any such restriction on $T$. However, since this setup requires the knowledge of $\tau_{\rm mix}$, one can treat $\tau_{\rm mix}$ as one of the unknown hyperparameters of the algorithm, and fine-tune it during the training phase. This is in line with other RL algorithms in the literature that also require fine-tuning of several unknown hyperparameters, such as the Lipschitz constant or hitting time \citep{bai2023regret}. Despite these practical solutions, we acknowledge that a systematic theoretical investigation is needed to improve the requirement of the horizon length, $T$ (in the absence of knowledge of $\tau_{\rm mix}$), which is left as one of the future works.

It is also worth highlighting that due to the presence of $\epsilon_{\mathrm{bias}}$ and $\epsilon_{\mathrm{app}}$, the average objective error and constraint violation cannot be guaranteed to be zero, even for large $T$. However, for rich policy parameterization and good critic approximation, the effects of these quantities are negligibly small.


\section{Conclusion}
In this paper, we investigate the infinite-horizon average reward Constrained Markov Decision Processes (CMDPs) with general policy parametrization. We propose a novel algorithm, the ``Primal-dual natural actor-critic," which efficiently manages constraints while achieving a global convergence rate of $\Tilde{\mathcal{O}}(1/\sqrt{T})$, aligning with the theoretical lower bound for Markov Decision Processes (MDPs). We also extend our analysis to the setting with unknown mixing time. Future directions include narrowing the performance gap in this setting, parameterizing the critic using neural networks as in \cite{gaur2024closing,ganesh2025order}, and relaxing the ergodicity assumption following \cite{ganesh2025regret}.

\section{Acknowledgment}
The work was supported in part by the National Science Foundation under  grant CCF-2149588, Office of Naval Research under grant N00014-23-1-2532, and Cisco Systems, Inc.

\bibliographystyle{plain}
\bibliography{main}


\appendix
\onecolumn
\section{Related Works} 

The constrained reinforcement learning problem has been widely explored for both infinite horizon discounted reward and episodic MDPs. Recent studies have examined discounted reward CMDPs in various contexts, including the tabular setting \citep{bai2022achieving}, with softmax parameterization \citep{ding2020natural,xu2021crpo}, and with general policy parameterization setting \citep{ding2020natural,xu2021crpo,bai2023achieving, mondal2024sample}. Additionally, episodic CMDPs have been explored in the tabular setting by \citep{efroni2020exploration, qiu2020upper, germano2023best}. Recent research has also focused on infinite horizon average reward CMDPs, examining various approaches including model-based setups  \citep{chen2022learning, agarwal2022regret,agarwal2022concave},  tabular model-free settings \citep{wei2022provably}, linear CMDP setting \citep{ghosh2022achieving} and general policy parametrization setting \citep{bai2024learning}. In model-based CMDP setups, \citep{agarwal2022concave} introduced algorithms leveraging posterior sampling and the optimism principle, achieving a convergence rate of $\tilde{\mathcal{O}}(1/\sqrt{T})$ with zero constraint violations. For tabular model-free approach, \citep{wei2022provably} attains a convergence rate of $\tilde{\mathcal{O}}(T^{-1/6})$ with zero constraint violations. In linear CMDP, \citep{ghosh2022achieving} obtains $\tilde{\mathcal{O}}(T^{-1/2})$ convergence rate with zero constraint violation. Note that linear CMDP assumes a linear structure in the transition probability based on a known feature map, which is not realistic. Finally, \citep{bai2024learning} studied the infinite horizon average reward CMDPs with general parametrization and achieved a global convergence rate of $\tilde{\mathcal{O}}(1/{T}^{1/5})$. Table \ref{table2} summarizes all relevant works on average reward CMDPs.

In unconstrained average reward MDPs, both model-based and model-free tabular setups have been widely studied. For instance, the model-based algorithms by \citep{agrawal2017optimistic, auer2008near} obtain the optimal convergence rate of $\tilde{\mathcal{O}}(1/\sqrt{T})$.
Similarly, the model-free algorithm proposed by \citep{wei2020model} achieves $\tilde{\mathcal{O}}(1/\sqrt{T})$ convergence rate for tabular MDP. Average reward MDP with general parametrization has been recently studied in \citep{ganesh2024accelerated,ganesh2024variance}, where global convergence rates of $\tilde{\mathcal{O}}(1/\sqrt{T})$ are achieved. In particular, \cite{ganesh2024accelerated} leverages Actor-Critic methods and achieves global convergence without knowledge of mixing time. 

\section{Preliminary Results for Global Convergence Analysis}

To prove Theorem \ref{th:acc_td} and \ref{thm:npg-main}, we first discuss a more general case of linear recursion with biased estimators.

\begin{theorem}
    \label{thm:generalrecursion}
    Consider the stochastic linear recursion that aims to approximate $x^*=P^{-1}q$.
    \begin{align}
        \label{eq:xt}
        x_{h+1} = x_h - \bar{\beta} (\hat{P}_h x_h - \hat{q}_h)
    \end{align}
    where $\hat{P}_h$, $\hat{q}_h$ are estimates of the matrices $P\in\mathbb{R}^{n\times n}$, $q\in\mathbb{R}^n$ respectively, and $h\in\{0, \cdots, H-1\}$. Assume that the following bounds hold $\forall h$.
    \begin{align}
        \label{eq_cond_1}
        \begin{split}
            &\E_h\left[\norm{\hat{P}_h-P}^2\right]\leq \sigma_P^2, ~ \norm{\E_h\left[\hat{P}_h\right]-P}^2\leq \delta^2_P,\\
            &\E_h\left[\norm{\hat{q}_h-q}^2\right]\leq \sigma_q^2, ~ \norm{\E_h\left[\hat{q}_h\right]-q}^2\leq \delta_q^2,~\text{and~}\norm{\E\left[\hat{q}_h\right]-q}^2\leq \bar{\delta}_q^2
        \end{split}
    \end{align}
    where $\E_h$ denotes conditional expectation given history up to step $h$. Since $\E[\hat{q}_h]=\E[\E_h[\hat{q}_h]]$, we have $\bar{\delta}^2_q\leq \delta_q^2$. Additionally, assume that
    \begin{align}
        \label{eq_cond_2}
        \begin{split}
            0<\lambda_P \leq \norm{P}\leq \Lambda_P~~\text{and}~\norm{q} \leq \Lambda_q
        \end{split}
    \end{align}
    The condition that $\lambda_P>0$ implies that $P$ is invertible. The goal of recursion \eqref{eq:xt} is to approximate the term $x^*=P^{-1}q$. If $\delta_P\leq \lambda_P/8$, and $\bar{\beta}\leq \lambda_P/[4(6\sigma_P^2+2\Lambda_P^2)]$, the following relation holds. 
    \begin{equation}
        \label{eq_appndx_lemma_exp_x_h_recursion}
        \E\left[\|x_H-x^*\|^2\right]\leq  \exp\left(-{H\bar{\beta} \lambda_P}\right)\E\|x_{0} -x^*\|^2 + \tilde{\cO}\bigg({\bar{\beta} R_0} + R_1\bigg)
    \end{equation}
    where $R_0 = \lambda_P^{-3}\Lambda_q^2\sigma_P^2+ \lambda_P^{-1}\sigma_q^2$, $R_1= \lambda_P^{-2}\big[\delta_P^2 \lambda_P^{-2}\Lambda_q^2+\delta_q^2\big]$, and $\tilde{\cO}(\cdot)$ hides logarithmic factors of $H$. Moreover,
    \begin{align}
        \|\E[x_H]-x^*\|^2 &\leq \exp\left(-{H\bar{\beta}\lambda_P}\right)\|\E[x_{0}] - x^*\|^2 \nonumber\\
        &+\dfrac{6}{\lambda_P}\left(\bar{\beta}+\dfrac{1}{\lambda_P}\right)\left[\delta_P^2\bigg\{\E\left[\|x_0 - x^*\|^2\right]+\mathcal{O}\left(\bar{\beta}R_0+R_1\right)\bigg\} + \bar{R}_1 \right]
    \end{align}
    where $\bar{R}_1 = \delta_P^2 \lambda_P^{-2}\Lambda_q^2+\bar{\delta}_q^2$.
\end{theorem}

\begin{proof}
    Let $g_h=\hat{P}_hx_h-\hat{q}_h$. To prove the first statement, observe the following relations.
    \begin{align*}
        &\|x_{h+1} - x^*\|^2 = \|x_{h} - \bar{\beta} {g}_h - x^*\|^2\\
        &= \|x_{h} -x^*\|^2- 2\bar{\beta} \langle x_{h} - x^*, {g}_h \rangle +\bar{\beta}^2\| g_h \|^2\\
        &\overset{}{=} \|x_{h} -x^*\|^2 - 2\bar{\beta} \langle x_{h} -x^*, P(x_h - x^*) \rangle - 2\bar{\beta} \langle x_{h} -x^*, {g}_h - P(x_h - x^*) \rangle + \bar{\beta}^2\| g_h \|^2\\
        &\overset{(a)}{\leq} \|x_{h} -x^*\|^2 - 2\bar{\beta}\lambda_P \| x_{h} -x^*\|^2 - 2\bar{\beta} \langle x_{h} -x^*, {g}_h - P(x_h - x^*) \rangle\\
        &\hspace{1cm}+ 2\bar{\beta}^2\| g_h - P(x_h - x^*)\|^2+ 2\bar{\beta}^2\| P(x_h - x^*)\|^2\\
        &\overset{(b)}{\leq} \|x_{h} -x^*\|^2 - 2\bar{\beta}\lambda_P \| x_{h} - x^*\|^2 - 2\bar{\beta} \langle x_{h} - x^*, g_h - P(x_h - x^*) \rangle\\
        &\hspace{1cm}+ 2\bar{\beta}^2\| g_h - P(x_h - x^*)\|^2+ 2\Lambda_{P}^2\bar{\beta}^2\|x_h - x^*\|^2
    \end{align*}
    where $(a)$, $(b)$ follow from $\lambda_P\leq \norm{P}\leq \Lambda_P$. Taking the conditional expectation $\E_h$ on both sides,
    \begin{align}
        \label{eq:mid-exp-td}
        \E_h\left[\left\|x_{h+1} - x^*\right\|^2\right] &\leq (1- 2\bar{\beta}\lambda_P +2\Lambda_{P}^2\bar{\beta}^2)\|x_{h} -x^*\|^2 \nonumber\\
        &- 2\bar{\beta} \langle x_{h} - x^*, \E_h \left[g_h - P(x_h - x^*)\right] \rangle + 2\bar{\beta}^2\E_h\left\| g_h - P(x_h - x^*)\right\|^2
    \end{align}
    Observe that the third term in \eqref{eq:mid-exp-td} can be bounded as follows.
    \begin{align*}
        \| g_h - P(x_h - x^*)\|^2 &= \| (\hat{P}_h - P)(x_h - x^*) + (\hat{P}_h - P)x^* + (q-\hat{q}_h)\|^2\\
        &\leq 3\| \hat{P}_h - P\|^2 \|x_h - x^*\|^2 + 3\|\hat{P}_h - P\|^2 \|x^*\|^2 + 3\|q-\hat{q}_h\|^2\\
        &\leq 3\| \hat{P}_h - P\|^2 \|x_h - x^*\|^2 + 3\lambda_P^{-2}\Lambda_q^2\|\hat{P}_h - P\|^2 + 3\|q-\hat{q}_h\|^2
    \end{align*}
    where the last inequality follows from $\norm{x^*}^2=\norm{P^{-1}q}^2\leq \lambda_P^{-2}\Lambda_q^2$. Taking the expectation yields
    \begin{align}
        \E_h&\| g_h - P(x_h - x^*)\|^2 \nonumber\\
        &\leq 3\E_h\| \hat{P}_h - P\|^2 \|x_h - x^*\|^2 + 3\lambda_P^{-2}\Lambda_q^2\E_h\|\hat{P}_h - P\|^2 + 3\E_h\|\hat{q}_h-q\|^2 \nonumber\\
        &\leq 3\sigma_P^2 \norm{x_h-x^*}^2 + 3\lambda_P^{-2}\Lambda_q^2\sigma_P^2 + 3\sigma_q^2
    \end{align}
    The second term in \eqref{eq:mid-exp-td} can be bounded as
    \begin{align}
        -\langle x_{h} - x^*, \E_h &\left[g_h - P(x_h - x^*)\right] \rangle \leq \frac{\lambda_P}{4} \| x_{h} - x^*\|^2 + \frac{1}{\lambda_P}\left\|\E_h [g_h - P(x_h - x^*)]\right\|^2 \nonumber\\
        &\leq \frac{\lambda_P}{4} \| x_{h} - x^*\|^2 + \dfrac{1}{\lambda_P}\left\|\left\{\E_h[\hat{P}_h]-P\right\}x_h + \bigg\{q-\E_h\left[\hat{q}_h\right]\bigg\}\right\|^2\nonumber\\
        &\leq \frac{\lambda_P}{4} \| x_{h} - x^*\|^2 + \frac{2\delta_P^2\|x_h\|^2+2\delta_q^2}{\lambda_P}\nonumber\\
        &\leq \frac{\lambda_P}{4} \| x_{h} - x^*\|^2 + \frac{4\delta_P^2\|x_h-x^*\|^2+4\delta_P^2\lambda_P^{-2}\Lambda_q^2 + 2\delta_q^2}{\lambda_P}
    \end{align}
    where the last inequality follows from $\norm{x^*}^2=\norm{P^{-1}q}^2\leq \lambda_P^{-2}\Lambda_q^2$. Substituting the above bounds in \eqref{eq:mid-exp-td}, we arrive at the following.
    \begin{align*}
        \E_h\left[\|x_{h+1} - x^*\|^2\right] &\leq \left(1- \frac{3\bar{\beta}\lambda_P}{2} +\dfrac{8\bar{\beta}\delta_P^2}{\lambda_P}+6\bar{\beta}^2\sigma_P^2+2\bar{\beta}^2\Lambda_P^2\right)\|x_{h} -x^*\|^2 \\
        &\hspace{1cm}+\dfrac{4\bar{\beta}}{\lambda_P}\left[2\delta_P^2 \lambda_P^{-2}\Lambda_q^2+\delta_q^2\right] + 6\bar{\beta}^2\left[\lambda_P^{-2}\Lambda_q^2\sigma_P^2+\sigma_q^2\right]
    \end{align*}
    For $\delta_P \leq \lambda_P/8$, and $\bar{\beta}\leq \lambda_P/[4(6\sigma_P^2+2\Lambda_P^2)]$, we can modify the above inequality to the following. 
    \begin{align*}
        &\E_h[\|x_{h+1} - x^*\|^2] \leq \left(1-{\bar{\beta} \lambda_P}\right)\|x_{h} -x^*\|^2 +\dfrac{4\bar{\beta}}{\lambda_P}\left[2\delta_P^2 \lambda_P^{-2}\Lambda_q^2+\delta_q^2\right]+6\bar{\beta}^2\left[\lambda_P^{-2}\Lambda_q^2\sigma_P^2+\sigma_q^2\right]
    \end{align*}
    Taking the expectation on both sides and unrolling the recursion yields
    \begin{align*} 
        \E&[\|x_{H} - x^*\|^2] \leq \left(1-{\bar{\beta} \lambda_P}\right)^H\E\|x_{0} -x^*\|^2 \\
        &\hspace{1cm}+\sum_{h=0}^{H-1} \left(1-{\bar{\beta} \lambda_P}\right)^{h}\left\{\dfrac{4\bar{\beta}}{\lambda_P}\left[2\delta_P^2 \lambda_P^{-2}\Lambda_q^2+\delta_q^2\right]+6\bar{\beta}^2\left[\lambda_P^{-2}\Lambda_q^2\sigma_P^2+\sigma_q^2\right]\right\}\\
        &\leq\exp\left(-{H\bar{\beta} \lambda_P}\right)\E\|x_{0} -x^*\|^2 + \frac{1}{\bar{\beta} \lambda_P}\left\{\dfrac{4\bar{\beta}}{\lambda_P}\left[2\delta_P^2 \lambda_P^{-2}\Lambda_q^2+\delta_q^2\right]+6\bar{\beta}^2\left[\lambda_P^{-2}\Lambda_q^2\sigma_P^2+\sigma_q^2\right]\right\}\\
        &=\exp\left(-{H\bar{\beta} \lambda_P}\right)\E\|x_{0} -x^*\|^2 + \bigg\{4\lambda_P^{-2}\left[2\delta_P^2 \lambda_P^{-2}\Lambda_q^2+\delta_q^2\right]+ 6\bar{\beta}\lambda_P^{-1}\left[\lambda_P^{-2}\Lambda_q^2\sigma_P^2+\sigma_q^2\right]\bigg\} \label{eq_lemma_x_H_recursion}
    \end{align*}

    To prove the second statement, observe that we have the following recursion.
    \begin{align}
        \|\E&[x_{h+1}] - x^*\|^2 = \|\E[x_{h}] - \bar{\beta} \E[{g}_h] - x^*\|^2\nonumber\\
        &=\|\E[x_{h}] -x^*\|^2- 2\bar{\beta} \langle \E[x_{h}] - x^*, \E[{g}_h] \rangle +\bar{\beta}^2\| \E[g_h] \|^2\nonumber\\
        &\overset{}{=} \|\E[x_{h}] -x^*\|^2 - 2\bar{\beta} \langle \E[x_{h}] -x^*, P(\E[x_h] - x^*) \rangle \nonumber\\
        &\hspace{1cm}- 2\bar{\beta} \langle \E[x_{h}] -x^*, \E[{g}_h] - P(\E[x_h] - x^*) \rangle + \bar{\beta}^2\| \E[g_h] \|^2\nonumber\\
        &\overset{(a)}{\leq} \|\E[x_{h}] -x^*\|^2 - 2\bar{\beta}\lambda_P \| \E[x_{h}]-x^*\|^2 - 2\bar{\beta} \langle \E[x_{h}] -x^*, \E[{g}_h] - P(\E[x_h] - x^*) \rangle\nonumber\\
        &\hspace{1cm}+ 2\bar{\beta}^2\| \E[g_h] - P(\E[x_h] - x^*)\|^2+ 2\bar{\beta}^2\| P(\E[x_h] - x^*)\|^2\nonumber\\
        &\overset{(b)}{\leq} \|\E[x_{h}] -x^*\|^2 - 2\bar{\beta}\lambda_P \| \E[x_{h}] - x^*\|^2 - 2\bar{\beta} \langle \E[x_{h}] - x^*, \E[g_h] - P(\E[x_h] - x^*) \rangle\nonumber\\
        &\hspace{1cm}+ 2\bar{\beta}^2\| \E[g_h] - P(\E[x_h] - x^*)\|^2+ 2\Lambda_{P}^2\bar{\beta}^2\|\E[x_h] - x^*\|^2\nonumber\\
        &\leq (1- 2\bar{\beta}\lambda_P +2\Lambda_{P}^2\bar{\beta}^2)\|\E[x_{h}] -x^*\|^2 - 2\bar{\beta} \langle \E[x_{h}] - x^*, \E[g_h] - P(\E[x_h] - x^*) \rangle \nonumber\\
        &\hspace{1cm}+ 2\bar{\beta}^2\left\| \E[g_h] - P(\E[x_h] - x^*)\right\|^2
        \label{eq_43_appndx_washim}
    \end{align}
    where $(a)$ and $(b)$ are consequences of $\lambda_P\leq \norm{P}\leq \Lambda_P$. The third term in the last line of \eqref{eq_43_appndx_washim} can be bounded as follows.
    \begin{align*}
        \| \E[g_h] &- P(\E[x_h] - x^*)\|^2 = \left\| \E\left[(\hat{P}_h - P)(x_h - x^*)\right] + (\E[\hat{P}_h] - P)x^* + (q-\E[\hat{q}_h])\right\|^2\\
        &\leq 3\E\left[\| \E_h[\hat{P}_h] - P\|^2 \|x_h - x^*\|^2\right] + 3\E\left[\|\E_h[\hat{P}_h] - P\|^2 \right]\|x^*\|^2 + 3\left\|q-\E[\hat{q}_h]\right\|^2\\
        &\leq 3\delta_P^2\E\left[\|x_h - x^*\|^2\right] + 3\lambda_P^{-2}\Lambda_q^2\delta_P^2 + 3\bar{\delta}_q^2\\
        &\overset{(a)}{\leq}3\delta_P^2\bigg\{\E\left[\|x_0 - x^*\|^2\right]+\mathcal{O}\left(\bar{\beta}R_0+R_1\right)\bigg\} + 3\lambda_P^{-2}\Lambda_q^2\delta_P^2 + 3\bar{\delta}_q^2
    \end{align*}
    where $(a)$ follows from \eqref{eq_appndx_lemma_exp_x_h_recursion}. The second term in the last line of \eqref{eq_43_appndx_washim} can be bounded as follows.
    \begin{align*}
        -\langle &\E[x_{h}] - x^*, \E_h \left[\E[g_h] - P(\E[x_h] - x^*)\right] \rangle \\
        &\leq\frac{\lambda_P}{4} \| \E[x_{h}] - x^*\|^2 + \frac{1}{\lambda_P}\left\|\E[g_h] - P(\E[x_h] - x^*)\right\|^2 \nonumber\\
        &\leq\frac{\lambda_P}{4} \|\E[x_{h}] - x^*\|^2 + \dfrac{3}{\lambda_P}\left[\delta_P^2\bigg\{\E\left[\|x_0 - x^*\|^2\right]+\mathcal{O}\left(\bar{\beta}R_0+R_1\right)\bigg\} + \lambda_P^{-2}\Lambda_q^2\delta_P^2 + \bar{\delta}_q^2\right]\nonumber
    \end{align*}

    Substituting the above bounds in \eqref{eq_43_appndx_washim}, we obtain the following recursion.
    \begin{align*}
        \|\E[x_{h+1}] &- x^*\|^2 \leq \left(1-\dfrac{3\bar{\beta}\lambda_P}{2}+2\Lambda_P^2\bar{\beta}^2\right)\|\E[x_{h}] - x^*\|^2 \\
        &\hspace{0.5cm}+6\bar{\beta}\left(\bar{\beta}+\dfrac{1}{\lambda_P}\right)\left[\delta_P^2\bigg\{\E\left[\|x_0 - x^*\|^2\right]+\mathcal{O}\left(\bar{\beta}R_0+R_1\right)\bigg\} + \lambda_P^{-2}\Lambda_q^2\delta_P^2 + \bar{\delta}_q^2\right]
    \end{align*}
    If $\bar{\beta}<\lambda_P/(4\Lambda_P^2)$, the above bound implies the following.
    \begin{align*}
        \|\E[x_{h+1}] &- x^*\|^2 \leq \left(1-{\bar{\beta}\lambda_P}\right)\|\E[x_{h}] - x^*\|^2 \\
        &+6\bar{\beta}\left(\bar{\beta}+\dfrac{1}{\lambda_P}\right)\left[\delta_P^2\bigg\{\E\left[\|x_0 - x^*\|^2\right]+\mathcal{O}\left(\bar{\beta}R_0+R_1\right)\bigg\} + \lambda_P^{-2}\Lambda_q^2\delta_P^2 + \bar{\delta}_q^2\right]
    \end{align*}
    Unrolling the above recursion, we obtain
    \begin{align*}
        &\|\E[x_{H}] - x^*\|^2 \leq \left(1-{\bar{\beta} \lambda_P}\right)^H\|\E[x_{0}] -x^*\|^2 \\
        &+\sum_{h=0}^{H-1} 6\left(1-{\bar{\beta} \lambda_P}\right)^{h}\bar{\beta}\left(\bar{\beta}+\dfrac{1}{\lambda_P}\right)\left[\delta_P^2\bigg\{\E\left[\|x_0 - x^*\|^2\right]+\mathcal{O}\left(\bar{\beta}R_0+R_1\right)\bigg\} + \lambda_P^{-2}\Lambda_q^2\delta_P^2 + \bar{\delta}_q^2\right]\\
        &\leq \exp\left(-{H\bar{\beta}\lambda_P}\right)\|\E[x_{0}] - x^*\|^2 \\
        &\hspace{1cm}+\dfrac{6}{\lambda_P}\left(\bar{\beta}+\dfrac{1}{\lambda_P}\right)\left[\delta_P^2\bigg\{\E\left[\|x_0 - x^*\|^2\right]+\mathcal{O}\left(\bar{\beta}R_0+R_1\right)\bigg\} + \bar{R}_1 \right]
    \end{align*}
    where $\bar{R}_1 = \delta_P^2 \lambda_P^{-2}\Lambda_q^2+\bar{\delta}_q^2$. This concludes the result.
\end{proof}

We now provide some bounds on MLMC estimates.

\begin{lemma}
    \label{lem:expect_bound_grad}
    Consider a time-homogeneous, ergodic Markov chain $(Z_t)_{t\geq 0}$ with a unique invariant distribution $d_Z$ and a mixing time $\tau_{\mathrm{mix}}$. Assume that $\nabla F(x, Z)$ is an estimate of the gradient $\nabla F(x)$. Let $\|\E_{d_Z}[\nabla F(x, Z)]-\nabla F(x)\|^2 \leq \delta^2 $ and $\|\nabla F(x, Z_t) - \E_{d_Z}\left[\nabla F(x, Z)\right]\|^2 \leq \sigma^2 $ for all $t\geq 0$. If $Q\sim \mathrm{Geom}(1/2)$, then the following MLMC estimator 
    \begin{align}
        g_{\mathrm{MLMC}} = g^0 +
        \begin{cases}
            2^{Q}\left(g^{Q} - g^{Q-1}\right), & \text{if}~ 2^{Q} \leq T_{\max} \\
            0, & \text{otherwise}
        \end{cases}
        \text{~~where}~g^j = 2^{-j} \sum\nolimits_{t=0}^{2^j-1} \nabla F(x, Z_{t})
    \end{align}
    satisfies the inequalities stated below.

    (a) $\E[g_{\mathrm{MLMC}}] = \E[g^{\lfloor \log T_{\max}\rfloor}]$

    (b) $\E[\| \nabla F(x) - g_{\mathrm{MLMC}}\|^2] \leq  \mathcal{O}\left(\sigma^2 \tau_{\mathrm{mix}}\log_2 T_{\max}+\delta^2\right)$

    (c) $\| \nabla F(x) - \E [g_{\mathrm{MLMC}}]\|^2 \leq \cO \left( \sigma^2 \tau_{\mathrm{mix}} T_{\max}^{-1} + \delta^2 \right)$
\end{lemma}

Before proceeding to the proof, we state a useful lemma.
\begin{lemma}[Lemma 1, \cite{beznosikov2023first}]
    \label{lem:tech_markov}
    Consider the same setup as in Lemma \ref{lem:expect_bound_grad}. The following holds.
    \begin{equation}
        \label{eq:var_bound_any}
        \E\left[\norm{\dfrac{1}{N}\sum\limits_{t=0}^{N-1}\nabla F(x, Z_t) - \E_{d_Z}\left[\nabla F(x, Z)\right]}^2\right] \leq \dfrac{C_{1} t_{\mathrm{mix}}}{N} \sigma^2,
    \end{equation}
    where $N$ is a constant, $C_{1} = 16(1 + \frac{1}{\ln^{2}{4}})$, and the expectation is over the distribution of $\{Z_t\}_{t=0}^{N-1}$ emanating from any arbitrary initial distribution.
\end{lemma}
\begin{proof}[Proof of Lemma \ref{lem:expect_bound_grad}] 
    The statement $(a)$ can be proven as follows.
    \begin{align*}
        \E[g_{\mathrm{MLMC}}] & =\E[g^0] + \sum\limits_{j=1}^{\lfloor \log_2 T_{\max} \rfloor} \Pr\{Q = j\} \cdot 2^j\E[g^{j} - g^{j-1}] \\
        &=\E[g^0] + \sum\limits_{j=1}^{\lfloor \log_2 T_{\max} \rfloor} \E[g^{j}  - g^{j-1}] = \E[g^{\lfloor \log_2 T_{\max} \rfloor}],
    \end{align*}
    For the proof of $(b)$, notice that 
    \begin{align*}
        &\E\left[\left\| \E_{d_Z}\left[\nabla F(x, Z)\right] - g_{\mathrm{MLMC}}\right\|^2\right]\leq 2\E\left[\left\| \E_{d_Z}\left[\nabla F(x_t)\right] - g^0\right\|^2\right] + 2\E\left[\left\| g_{\mathrm{MLMC}} - g^0\right\|^2\right]\\
        &=2\E\left[\left\| \E_{d_Z}\left[\nabla F(x_t)\right] - g^0\right\|^2\right] + 2 \sum\nolimits_{j=1}^{\lfloor \log_2 T_{\max} \rfloor} \Pr\{Q = j\} \cdot 4^j \E\left[\left\|g^{j}  - g^{j-1}\right\|^2\right] \\
        &=2\E\left[\left\| \E_{d_Z}\left[\nabla F(x_t)\right] - g^0\right\|^2\right] + 2\sum\nolimits_{j=1}^{\lfloor \log_2 T_{\max} \rfloor} 2^j \E\left[\left\|g^{j}  - g^{j-1}\right\|^2\right] \\
        &\leq 2\E\left[\left\| \E_{d_Z}\left[\nabla F(x_t)\right] - g^0\right\|^2\right] \\
        &\hspace{1cm}+ 4\sum\nolimits_{j=1}^{\lfloor \log_2 T_{\max} \rfloor} 2^j \left(\E\left[\left\|\E_{d_z}\left[\nabla F(x, Z)\right]  - g^{j-1}\right\|^2\right] + \E\left[\left\|g^{j}  - \E_{d_Z}\left[\nabla F(x, Z)\right]\right\|^2\right] \right)\,\\
        &\overset{(a)}{\leq} C_{1} t_{\mathrm{mix}} \sigma^2 \left[2+ 4\sum\nolimits_{j=1}^{\lfloor\log_2 T_{\max}\rfloor} 2^j\left(\dfrac{1}{2^{j-1}}+\dfrac{1}{2^{j}}\right)\right]=\mathcal{O}\left(\sigma^2t_{\mathrm{mix}}\log_2 T_{\max} \right)
    \end{align*}
    where $(a)$ follows from Lemma \ref{lem:tech_markov}. Using this result, we obtain the following.
    \begin{align*}
        \E\left[\left\|\nabla F(x) - g_{\mathrm{MLMC}} \right\|^2\right]&\leq 2\E\left[\left\|\nabla F(x) - \E_{d_z}\left[\nabla F(x, Z)\right]\right\|^2\right] + 2\E\left[\left\|\E_{d_z}\left[\nabla F(x, Z)\right]-g_{\mathrm{MLMC}}\right\|^2\right]\\
        &\leq \mathcal{O}\left(\sigma^2 t_{\mathrm{mix}}\log_2 T_{\max}+\delta^2\right)
    \end{align*}

    This completes the proof of statement $(b)$. For part $(c)$, we have
    \begin{align}
        \begin{split}
            &\left\|\nabla F(x) - \E\left[g_{\mathrm{MLMC}}\right]\right\|^2\\
            &\leq 2\left\|\nabla F(x) - \E_{d_Z}\left[\nabla F(x, Z)\right] \right\|^2 + 2\left\| \E_{d_Z}\left[\nabla F(x, Z)\right] - \E\left[g_{\mathrm{MLMC}}\right]\right\|^2\\
            &\leq 2\delta^2 + 2\left\|  \E_{d_Z}\left[\nabla F(x, Z)\right] - \E[g^{\lfloor \log_2 T_{\max} \rfloor}]\right\|^2 \overset{(a)}{\leq} 2\delta^2+ \frac{2C_1 t_{\mathrm{mix}}}{T_{\max}} \sigma^2 
        \end{split}
    \end{align}

    where $(a)$ follows from Lemma \ref{lem:tech_markov}. This concludes the proof of Lemma \ref{lem:expect_bound_grad}.
\end{proof}


\section{Proof of Lemma \ref{lemma:local_global}}

We begin by stating a useful lemma:
\begin{lemma}[Lemma 4, \citep{bai2023regret}]
    \label{lem_performance_diff}
    The performance difference between two policies $\pi_\theta$, $\pi_{\theta'}$ is bounded as follows where $g\in \{r, c\}$.
    \begin{equation}
        J(\theta)-J(\theta')= \E_{s\sim d^{\pi_\theta}}\E_{a\sim\pi_\theta(\cdot\vert s)}\big[A^{\pi_{\theta'}}(s,a)\big].
    \end{equation}
\end{lemma}
Continuing with the proof of Lemma \ref{lemma:local_global}, we start with the definition of KL divergence. For notational simplicity, we shall use $A^{\pi_{\theta_k}}_{r+\lambda_k c}$ to denote $A^{\pi_{\theta_k}}_r + \lambda_k A^{\pi_{\theta_k}}_c$.
\begin{align*}
    &\E_{s\sim d^{\pi^*}}[KL(\pi^*(\cdot\vert s)\Vert\pi_{\theta_k}(\cdot\vert s))-KL(\pi^*(\cdot\vert s)\Vert\pi_{\theta_{k+1}}(\cdot\vert s))]=\E_{(s, a)\sim \nu^{\pi^*}}\bigg[\log\frac{\pi_{\theta_{k+1}}(a\vert s)}{\pi_{\theta_k}(a\vert s)}\bigg]\\
    &\overset{(a)}\geq\E_{(s, a)\sim \nu^{\pi^*}}[\nabla_\theta\log\pi_{\theta_k}(a\vert s)\cdot(\theta_{k+1}-\theta_k)]-\frac{G_2}{2}\Vert\theta_{k+1}-\theta_k\Vert^2\\
    &=\alpha\E_{(s, a) \sim \nu^{\pi^*}}[\nabla_{\theta}\log\pi_{\theta_k}(a\vert s)\cdot \omega_k]-\frac{G_2\alpha^2}{2}\Vert\omega_k\Vert^2\\
    &=\alpha\E_{(s, a) \sim \nu^{\pi^*}}[\nabla_\theta\log\pi_{\theta_k}(a\vert s)\cdot\omega^*_k]+\alpha\E_{(s, a)\sim \nu^{\pi^*}}[\nabla_\theta\log\pi_{\theta_k}(a\vert s)\cdot(\omega_k-\omega^*_k)]-\frac{G_2\alpha^2}{2}\Vert\omega_k\Vert^2\\
    &=\alpha[\mathcal{L}(\pi^*,\lambda_k)-\mathcal{L}(\theta_k,\lambda_k)]+\alpha\E_{(s, a) \sim \nu^{\pi^*}}[\nabla_\theta\log\pi_{\theta_k}(a\vert s)\cdot\omega^*_k]\\
    &-\alpha[\mathcal{L}(\pi^*,\lambda_k)-\mathcal{L}(\theta_k,\lambda_k)]+\alpha\E_{(s, a)\sim \nu^{\pi^*}}[\nabla_\theta\log\pi_{\theta_k}(a\vert s)\cdot(\omega_k-\omega^*_k)]-\frac{G_2\alpha^2}{2}\Vert\omega_k\Vert^2\\		
    &\overset{(b)}=\alpha[\mathcal{L}(\pi^*,\lambda_k)-\mathcal{L}(\theta_k,\lambda_k)]+\alpha\E_{(s, a) \sim \nu^{\pi^*}}\bigg[\nabla_\theta\log\pi_{\theta_k}(a\vert s)\cdot\omega^*_k- A_{r+\lambda_k c}^{\pi_{\theta_k}}(s,a)\bigg]\\
    &+\alpha\E_{(s, a)\sim \nu^{\pi^*}}[\nabla_\theta\log\pi_{\theta_k}(a\vert s)\cdot(\omega_k-\omega^*_k)]-\frac{G_2\alpha^2}{2}\Vert\omega_k\Vert^2\\
    &\overset{(c)}\geq\alpha[\mathcal{L}(\pi^*,\lambda_k)-\mathcal{L}(\theta_k,\lambda_k)]-\alpha\sqrt{\E_{(s, a)\sim \nu^{\pi^*}}\bigg[\bigg(\nabla_\theta\log\pi_{\theta_k}(a\vert s)\cdot\omega^*_k- A_{r+\lambda_k c}^{\pi_{\theta_k}}(s,a)\bigg)^2\bigg]}\\
    &+\alpha\E_{(s, a)\sim \nu^{\pi^*}}[\nabla_\theta\log\pi_{\theta_k}(a\vert s)\cdot(\omega_k-\omega^*_k)]-\frac{G_2\alpha^2}{2}\Vert\omega_k\Vert^2\\
    &\overset{(d)}\geq\alpha[\mathcal{L}(\pi^*,\lambda_k)-\mathcal{L}(\theta_k,\lambda_k)]-\alpha\sqrt{\epsilon_{\mathrm{bias}}} +\alpha\E_{(s, a)\sim \nu^{\pi^*}}[\nabla_\theta\log\pi_{\theta_k}(a\vert s)\cdot(\omega_k-\omega^*_k)]-\frac{G_2\alpha^2}{2}\Vert\omega_k\Vert^2\\
\end{align*}	
where $\mathcal{L}(\pi^*,\lambda_k) \coloneqq J_r^{\pi^*} +\lambda_k J_c^{\pi^*}$, the relations $(a)$, $(b)$ result from Assumption \ref{assump:score_func_bounds}$(b)$ and Lemma \ref{lem_performance_diff}, respectively. Inequality $(c)$ arises from the convexity of the function $f(x)=x^2$. Lastly, $(d)$ is a consequence of Assumption \ref{assump:function_approx_error}. By taking expectations on both sides, we derive:
\begin{align}
    \begin{split}
        &\E\left[\E_{s\sim d^{\pi^*}}\left[KL(\pi^*(\cdot\vert s)\Vert\pi_{\theta_k}(\cdot\vert s))-KL(\pi^*(\cdot\vert s)\Vert\pi_{\theta_{k+1}}(\cdot\vert s))\right]\right]\\
        &\geq \alpha[\mathcal{L}(\pi^*,\lambda_k)- \E[\mathcal{L}(\theta_k,\lambda_k)]]-\alpha\sqrt{\epsilon_{\mathrm{bias}}}\\
        &\hspace{1cm}+\alpha\E\left[\E_{s\sim d^{\pi^*}}\E_{a\sim\pi^*(\cdot\vert s)}[\nabla_\theta\log\pi_{\theta_k}(a\vert s)\cdot(\E[\omega_k|\theta_k]-\omega^*_k)]\right]-\frac{G_2\alpha^2}{2}\E\left[\Vert\omega_k\Vert^2\right]\\
        &\overset{}{\geq } \alpha[\mathcal{L}(\pi^*,\lambda_k)- \E[\mathcal{L}(\theta_k,\lambda_k)]]-\alpha\sqrt{\epsilon_{\mathrm{bias}}}\\
        &\hspace{1cm}-\alpha\E\left[\E_{(s, a)\sim \nu^{\pi^*}}[\Vert \nabla_\theta\log\pi_{\theta_k}(a\vert s)\Vert \Vert\E[\omega_k|\theta_k]-\omega^*_k\Vert]\right]-\frac{G_2\alpha^2}{2}\E\left[\Vert\omega_k\Vert^2\right]\\
        &\overset{(a)}{\geq } \alpha[\mathcal{L}(\pi^*,\lambda_k)- \E[\mathcal{L}(\theta_k,\lambda_k)]]-\alpha\sqrt{\epsilon_{\mathrm{bias}}} \\
        &\hspace{1cm}-\alpha G_1\E \Vert(\E[\omega_k|\theta_k]-\omega^*_k)\Vert-\frac{G_2\alpha^2}{2}\E\left[\Vert\omega_k\Vert^2\right]
    \end{split}
\end{align}
where $(a)$ follows from Assumption \ref{assump:score_func_bounds}$(a)$. Rearranging the terms, we get,
\begin{equation}
    \begin{split}
        \mathcal{L}(\pi^*,\lambda_k)- &\E[\mathcal{L}(\theta_k,\lambda_k)]\leq \sqrt{\epsilon_{\mathrm{bias}}}+ G_1\E\Vert(\E[\omega_k|\theta_k]-\omega^*_k)\Vert+\frac{G_2\alpha}{2}\E\Vert\omega_k\Vert^2\\
        &+\frac{1}{\alpha}\E\left[\E_{s\sim d^{\pi^*}}[KL(\pi^*(\cdot\vert s)\Vert\pi_{\theta_k}(\cdot\vert s))-KL(\pi^*(\cdot\vert s)\Vert\pi_{\theta_{k+1}}(\cdot\vert s))]\right]
    \end{split}
\end{equation}
Adding the above inequality from $k=0$ to $K-1$, using the non-negativity of KL divergence and dividing the resulting expression by $K$, we obtain the final result.


\section{Proof of Theorem \ref{th:acc_td}}

Recall that $\mathbf{A}_g(\theta) = \E_{\theta} [\mathbf{A}_g(\theta;z)]$, $\mathbf{b}_g(\theta) = \E_{\theta} [\mathbf{b}_g(\theta;z)]$, and $\xi^{*}_{g}(\theta) = (\mathbf{A}_{g}(\theta))^{-1}\mathbf{b}_{g}(\theta)$ where $\E_{\theta}$ denotes the expectation over the distribution of $z=(s, a, s')$ where $(s, a)\sim \nu^{\pi_{\theta}}$, $s'\sim P(s, a)$ (refer to section \ref{sec_global_conv_ananalysis}).
\begin{lemma}
    \label{lem:td-pd}
    For large $c_{\gamma}$, Assumption \ref{assum:critic_positive_definite} implies that $\mathbf{A}_g(\theta_k)-(\lambda/2)I$ is positive semi-definite for both $g \in \{r,c\}$ and for all $k$ where $I$ is an identity matrix, i.e., $\xi^\top \mathbf{A}_g(\theta_k) \xi \geq \lambda/2\cdot \|\xi\|^2$, for all $\xi$.
\end{lemma}
\begin{proof}[Proof of Lemma \ref{lem:td-pd}]
    Note that for any $\xi=[\eta, \zeta]$, we have
    \begin{align}
        \begin{split}
            \xi^{\top}\mathbf{A}_g(\theta_k)\xi &= c_{\gamma}\eta^2 + \eta\zeta^{\top}\E_{\theta_k}\left[\phi_g(s)\right] + \zeta^{\top}\E_{\theta_k}\left[\phi_g(s)\left[\phi_g(s)-\phi_g(s')\right]^{\top}\right]\zeta\\
            &\overset{(a)}{\geq} c_{\gamma}\eta^2 - |\eta|\norm{\zeta} + \lambda \norm{\zeta}^2\\
            &\geq\norm{\xi}^2\left\{\min_{u\in[0, 1]} c_{\gamma} u - \sqrt{u(1-u)}+\lambda(1-u)\right\}\overset{(b)}{\geq}(\lambda/2)\norm{\xi}^2
        \end{split}
    \end{align}
    where $(a)$ is a consequence of Assumption \ref{assum:critic_positive_definite} and the fact that $\norm{\phi_g(s)}\leq 1$, $\forall s\in\mathcal{S}$. Finally, $(b)$ holds when $c_{\gamma} \geq \lambda + \sqrt{\frac{1}{\lambda^2}-1}$. This completes the proof.
\end{proof}

 Let $\mathbf{A}^{\mathrm{MLMC}}_{g,kh}$ and $\mathbf{b}^{\mathrm{MLMC}}_{g,kh}$ be the MLMC estimators of $\{\mathbf{A}_g(\theta_k; z_{kh}^t)\}_{t=0}^{l_{kh}-1}$ and $\{\mathbf{b}_g(\theta_k; z_{kh}^t)\}_{t=0}^{l_{kh}-1}$ respectively (see \eqref{eq_19_washim}) i.e.
\begin{equation}
    \mathbf{A}^{\mathrm{MLMC}}_{g,kh} = \mathbf{A}_{g,kh}^0 +
    \begin{cases}
        2^{P_h^k}(\mathbf{A}_{g,kh}^{Q_h^k} - \mathbf{A}_{g,kh}^{Q_h^k - 1}),& \text{if } 2^{Q_h^k}\leq T_{\max}\\
        0, &\text{otherwise}
    \end{cases}             
\end{equation}
where $\mathbf{A}_{g,kh}^j = \frac{1}{2^j}\sum_{t=0}^{2^j-1}\mathbf{A}_g (\theta_k;z_{kh}^t)$ and 
\begin{equation}
    \mathbf{b}^{\mathrm{MLMC}}_{g,kh} = \mathbf{b}_{g,kh}^0 +
    \begin{cases}
        2^{P_h^k}(\mathbf{b}_{g,kh}^{Q_h^k} - \mathbf{b}_{g,kh}^{Q_h^k - 1}),& \text{if } 2^{Q_h^k}\leq T_{\max}\\
        0, & \text{otherwise}
    \end{cases}             
\end{equation}
where $\mathbf{b}_{g,kh}^j = \frac{1}{2^j}\sum_{t=0}^{2^j-1}\mathbf{b}_g (\theta_k;z_{kh}^t)$.

We can bound the bias and variance of $\mathbf{A}^{\mathrm{MLMC}}_{g,kh}$ and $\mathbf{b}^{\mathrm{MLMC}}_{g,kh}$ as follows.
\begin{lemma}
    \label{lem:A-b-bias}
    Consider Algorithm \ref{alg:pdranac} with a policy parameter $\theta_k$ and assume assumptions \ref{assump:ergodic_mdp}-\ref{assump:FND_policy} hold. The MLMC estimators $\mathbf{A}^{\mathrm{MLMC}}_{g,kh}$ and $\mathbf{b}^{\mathrm{MLMC}}_{g,kh}$ obey the following bounds.
    \begin{itemize}
        \item[(a)] $\|\E_{k,h}[\mathbf{A}^{\mathrm{MLMC}}_{g,kh}]-\mathbf{A}_g(\theta_k)\|^2 \leq \cO\left(c_\gamma^2\tau_{\mathrm{mix}}T_{\max}^{-1}\right)$
        \item[(b)] $\|\E_{k,h}[\mathbf{b}^{\mathrm{MLMC}}_{g,kh}]-\mathbf{b}_g(\theta_k)\|^2 \leq 
        \cO\left(c_\gamma^2 \tau_{\mathrm{mix}}T_{\max}^{-1}\right)$
        \item[(c)] $\E_{k,h} \big[ \|\mathbf{A}^{\mathrm{MLMC}}_{g,kh}-\mathbf{A}_g(\theta_k)\|^2 \big] \leq \cO \left(c_\gamma^2\tau_{\mathrm{mix}} \log T_{\max}\right)$
        \item[(d)] $\E_{k,h} \big[\|\mathbf{b}^{\mathrm{MLMC}}_{g,kh}-\mathbf{b}_g(\theta_k)\|^2 \big] \leq \cO \left(c_\gamma^2\tau_{\mathrm{mix}} \log T_{\max}\right)$
    \end{itemize}
    where $h\in\{0, \cdots, H-1\}$ and $\E_{k,h}$ defines the conditional expectation given the history up to the inner loop step $h$ of the critic within the $k$th outer loop instant.
\end{lemma}
\begin{proof}
    Recall from \eqref{eq_19_washim} the definitions of $\mathbf{A}_g(\theta; z)$ and $\mathbf{b}(\theta_k; z)$ for any $z=(s, a, s')\in\mathcal{S}\times\mathcal{A}\times\mathcal{S}$. We have the following inequalities. 
    \begin{align}
    \label{eq_appndx_washim_54}
        &\norm{\mathbf{A}_g(\theta_k;z)}\leq |c_{\gamma}|+\norm{\phi_g(s)} + \norm{\phi_g(s)(\phi_g(s)-\phi_g(s'))^{\top}}\overset{(a)}{\leq} c_{\gamma}+3=\mathcal{O}(c_{\gamma}),\\
        \label{eq_appndx_washim_55}
        &\norm{\mathbf{b}_g(\theta_k;z)} \leq |c_{\gamma}g(s, a)| + \norm{g(s, a)\phi_g(s)} \overset{(b)}{\leq} c_{\gamma}+1 =\mathcal{O}(c_{\gamma})
    \end{align}
    where $(a)$, $(b)$ hold since $|g(s, a)|\leq 1$ and $\norm{{\phi_g(s)}}\leq 1$, $\forall (s, a)\in\mathcal{S}\times\mathcal{A}$, $\forall g \in \{r,c\}$. Hence, for any $z^j_{kh}\in\mathcal{S}\times\mathcal{A}\times\mathcal{S}$, we have the inequalities stated below.
    \begin{align*}
        \norm{\mathbf{A}_g(\theta_k;z^j_{kh})- \mathbf{A}_g(\theta_k)}^2 \leq \mathcal{O}(c_{\gamma}^2), ~\text{and}~ \norm{\mathbf{b}_g(\theta_k;z^j_{kh})- \mathbf{b}_g(\theta_k)}^2 \leq \mathcal{O}(c_{\gamma}^2)
    \end{align*}

    Combining the above results with Lemma \ref{lem:expect_bound_grad}, and using the definitions that $\mathbf{A}_g(\theta) = \E_{\theta} [\mathbf{A}_g(\theta;z)]$, $\mathbf{b}_g(\theta) = \E_{\theta} [\mathbf{b}_g(\theta;z)]$, we establish the result.
\end{proof}

Combining Lemma \ref{lem:td-pd} with \eqref{eq_appndx_washim_54}, \eqref{eq_appndx_washim_55}, we obtain the following for any $\theta_k$ with $c_\gamma  \geq \lambda + \sqrt{\frac{1}{\lambda^2}-1}$.
\begin{align}
    \label{eq_imp_appndx_58}
    \dfrac{\lambda}{2}\leq \norm{\mathbf{A}_g(\theta_k)}\leq \cO(c_{\gamma}), ~\text{and}~ \norm{\mathbf{b}_g(\theta_k)}\leq \cO(c_{\gamma})
\end{align}

Combining the above results with Lemma \ref{lem:td-pd} along with Theorem \ref{thm:generalrecursion}, we then obtain the following inequalities given that $T_{\max} \geq \frac{8c_\gamma^2 \tau_{\mathrm{mix}}}{\lambda}$ and $\gamma_\xi \leq \frac{\lambda}{24 c_\gamma^2\tau_{\mathrm{mix}}  \log T_{\max}}$.
\begin{align*}
    &\E_k\left[\norm{\xi_{g,H}^k-\xi_g^{*}(\theta_k)}^2\right]\leq \exp\left(-{H \gamma_\xi \lambda}\right)\norm{\xi_0 - \xi^{*}_{g}(\theta_k)}^2 + \Tilde{\cO}\left(\gamma_\xi R_0+R_1\right), \\
    &\E_k\left[\norm{\E_k[\xi_{g,H}^k]-\xi_g^{*}(\theta_k)}^2\right]\leq \exp\left(-{H \gamma_\xi \lambda}\right)\norm{\xi_0 - \xi^{*}_{g}(\theta_k)}^2\\ &\hspace{1cm}+\dfrac{\lambda\gamma_\xi+1}{\lambda^2}\left[\mathcal{O}(T_{\max}^{-1}c_{\gamma}^2\tau_{\mathrm{mix}})\bigg\{\norm{\xi_0 - \xi^{*}_{g}(\theta_k)}^2+\mathcal{O}\left(\gamma_\xi R_0+R_1\right)\bigg\} + \bar{R}_1 \right]
\end{align*}
where the terms $R_0, R_1, \bar{R}_1$ are defined as follows.
\begin{align*}
    &R_0 = \tilde{\cO}\left(\lambda^{-3}c_{\gamma}^4 \tau_{\mathrm{mix}} + \lambda^{-1}c_{\gamma}^2\tau_{\mathrm{mix}} \right) = \tilde{\cO}\left(\lambda^{-3}c_{\gamma}^4 \tau_{\mathrm{mix}}\right),\\
    &R_1=\cO\left( T_{\max}^{-1}\lambda^{-4}c_{\gamma}^4\tau_{\mathrm{mix}}+T_{\max}^{-1}\lambda^{-2}c_{\gamma}^2\tau_{\mathrm{mix}}\right)=\cO\left(T_{\max}^{-1}\lambda^{-4}c_{\gamma}^4\tau_{\mathrm{mix}}\right)\\
    &\bar{R}_1=\cO\left(T_{\max}^{-1}\lambda^{-2}c_{\gamma}^4t_{\mathrm{mix}}+T_{\max}^{-1}c_{\gamma}^2t_{\mathrm{mix}}\right)=\cO\left(T_{\max}^{-1}\lambda^{-2}c_{\gamma}^4 \tau_{\mathrm{mix}}\right)
\end{align*}

If we further set $\gamma_\xi = \frac{2 \log T}{\lambda H}$ while ensuring $\frac{2 \log T}{\lambda H} \leq \frac{\lambda}{24 c_\gamma^2\tau_{\mathrm{mix}}  \log T_{\max}} $, we have the following results.
\begin{align*}
    &\E_k\left[\norm{\xi_{g,H}^k-\xi^{*}_{g}(\theta_k)}^2\right]\leq \dfrac{1}{T^2}\norm{\xi_0 - \xi^{*}_{g}(\theta_k)}^2 + \tilde{\cO}\left(\lambda^{-4}c_{\gamma}^4 \tau_{\mathrm{mix}}\left(\dfrac{1}{ H}+\dfrac{1}{T_{\max}}\right)\right), \\
    &\norm{\E_k[\xi_{g,H}^k]-\xi^{*}_{g}(\theta_k)}^2\leq \tilde{\cO}\left(\left(\dfrac{1}{T^2}+\dfrac{c_{\gamma}^2\tau_{\mathrm{mix}}}{\lambda^{2} T_{\max}}\right)\norm{\xi_0 - \xi^{*}_{g}(\theta_k)}^2\right)+ \tilde{\cO}\left(\dfrac{c_{\gamma}^6\tau^2_{\mathrm{mix}}}{\lambda^{6} T_{\max}}\right) 
\end{align*}
We used the fact that $H^{-1}+T_{\max}^{-1}=\mathcal{O}(1)$ in the last inequality. Utilizing $\xi_{g, H}^k=\xi_g^k$, $\xi_0=0$, and $\|\xi_g^*(\theta_k)\|=\|[\mathbf{A}_g(\theta_k)]^{-1}\mathbf{b}_g(\theta_k)\|=\mathcal{O}(\lambda^{-2}c_{\gamma}^2)$ (via \eqref{eq_imp_appndx_58}), we conclude:
\begin{align}              
    \label{eq_appndx_56_washim}
    &\E_k\left[\norm{\xi_{g}^k-\xi^{*}_{g}(\theta_k)}^2\right]\leq \dfrac{c_{\gamma}^2}{\lambda^2T^2}+ \tilde{\cO}\left(\lambda^{-4}c_{\gamma}^4 \tau_{\mathrm{mix}}\left(\dfrac{1}{ H}+\dfrac{1}{T_{\max}}\right)\right), \\
    \label{eq_appndx_washim_57}
    &\norm{\E_k[\xi_{g}^k]-\xi^{*}_{g}(\theta_k)}^2\leq \tilde{\cO}\left(\left(\dfrac{1}{T^2}+\dfrac{c_{\gamma}^2\tau_{\mathrm{mix}}}{\lambda^{2} T_{\max}}\right)\frac{c_{\gamma}^2}{\lambda^2}+\dfrac{c_{\gamma}^6\tau^2_{\mathrm{mix}}}{\lambda^{6} T_{\max}}\right) = \tilde{\cO}\left(\frac{c_{\gamma}^2}{\lambda^2 T^2}+ \frac{c_{\gamma}^6\tau^2_{\mathrm{mix}}}{\lambda^6 T_{\max}}\right)
\end{align}
This completes the proof.


\section{Proof of Theorem \ref{thm:npg-main}}
\label{sec:npg-app}

To prove Theorem \ref{thm:npg-main}, we follow a proof structure similar to that of Theorem \ref{th:acc_td}. Let $F^{\mathrm{MLMC}}_{kh}$ and $\nabla^{\mathrm{MLMC}}_\theta J_{g,kh}$ denote the MLMC estimators defined as follows.
\begin{equation}
    F^{\mathrm{MLMC}}_{kh} = F_{kh}^0 +
    \begin{cases}
        2^{Q_h^k}(F_{kh}^{Q_h^k} - F_{kh}^{Q_h^k - 1}),& \text{if } 2^{Q_h^k}\leq T_{\max}\\
        0, & \text{otherwise}
    \end{cases}             
\end{equation}
with $F_{kh}^j = \frac{1}{2^j}\sum_{t=0}^{2^j-1}F (\theta_k;z_{kh}^t)$ (see \eqref{eq_22_washim_new}) and 
\begin{equation}
    \nabla^{\mathrm{MLMC}}_\theta J_{g,kh} = \nabla^{0}_\theta J_{g,kh}^0 +
    \begin{cases}
        2^{Q_h^k}(\nabla^{Q_h^k}_\theta J_{g,kh} - \nabla^{Q_h^k - 1}_\theta J_{g,kh}), & \text{if } 2^{Q_h^k}\leq T_{\max}\\
        0, & \text{otherwise}
    \end{cases}             
\end{equation}
where $\nabla^{j}_\theta J_{g,kh} = 2^{-j}\sum_{t=0}^{2^j-1}\hat{\nabla}_\theta J_g (\theta_k, \xi_g^k;z_{kh}^t)$ (see \eqref{eq_22_washim_new}). 

Using the above inequalities, we can bound the bias and variance of the MLMC estimators as follows.
\begin{lemma}
    \label{lem:F-J-bias}
    Consider Algorithm \ref{alg:pdranac} with given policy parameter $\theta_k$. Under the assumptions \ref{assump:ergodic_mdp}-\ref{assump:FND_policy}, the following statements hold.
    \begin{itemize}
        \item[(a)] $\|\E_{k,h}[F^{\mathrm{MLMC}}_{kh}]-F(\theta_k)\|^2 \leq \cO\left(G_1^4\tau_{\mathrm{mix}}T_{\max}^{-1}\right)$
        \item[(b)] $\|\E_{k,h}[\nabla^{\mathrm{MLMC}}_\theta J_{g,kh}]-\nabla_\theta J_g(\theta_k)\|^2 \leq \cO\left(\sigma^2_{k,g}\tau_{\mathrm{mix}}T_{\max}^{-1} + \delta^2_{k,g}\right)$
        \item[(c)] $\E_{k,h}\big[\|F^{\mathrm{MLMC}}_{kh}-F(\theta_k)\|^2\big] \leq \cO \left(G_1^4 \tau_{\mathrm{mix}} \log T_{\max}\right)$
        \item[(d)] $\E_{k,h}\big[\|\nabla^{\mathrm{MLMC}}_\theta J_{g,kh}-\nabla_\theta J_g(\theta_k)\|^2\big] \leq \cO \left(\sigma^2_{k,g} \tau_{\mathrm{mix}} \log T_{\max}+ \delta^2_{k,g}\right)$.
        \item[(e)] $\|\E_{k}[\nabla^{\mathrm{MLMC}}_\theta J_{g,kh}]-\nabla_\theta J_g(\theta_k)\|^2 \leq \cO\left(\bar{\sigma}^2_{k,g}\tau_{\mathrm{mix}}T_{\max}^{-1} + \bar{\delta}^2_{k,g}\right)$ 
    \end{itemize}
    where $\E_{k,h}$ defines the conditional expectation given the history up to the inner loop step $h$ of the actor within the $k$th outer loop instant, while $\E_k$ is the conditional expectation given the history up to the $k$th outer loop instant. Moreover, 
    \begin{align*}
        &\sigma_{k,g}^2=\cO\left(G_1^2\norm{\xi^k_{g}}^2\right),\\
        &\bar{\sigma}_{k,g}^2=\cO\left(G_1^2\E_{k}\big[\norm{\xi^k_{g}}^2\big]\right),\\ &\delta_{k,g}^2=\cO\left(G_1^2\norm{\xi^k_{g} - \xi_g^{*}(\theta_k)}^2 + G_1^2\epsilon_{\mathrm{app}}\right),\\ &\bar{\delta}_{k,g}^2=\cO\left(G_1^2\norm{\E_{k}[\xi^k_{g}] - \xi_g^{*}(\theta_k)}^2 + G_1^2\epsilon_{\mathrm{app}}\right) 
    \end{align*}
    for $h\in\{0, 1, \cdots, H-1\}$.
\end{lemma}
\begin{proof}[Proof of Lemma \ref{lem:F-J-bias}]
    Fix an outer loop instant $k$ and an inner loop instant $h\in\{0, \cdots, H-1\}$ of the actor subroutine. Recall the definition of ${F}(\theta_k; \cdot)$ from \eqref{eq_22_washim_new}. The following inequalities hold for any $\theta_k$ and $z^j_{kh}\in\mathcal{S}\times \mathcal{A} \times \mathcal{S}$.
    \begin{align*}
        &\E_{\theta_k}\left[{F}(\theta_k; z) \right] \overset{(a)}{=} F(\theta_k), ~\text{and}~ \norm{{F}(\theta_k; z^j_{kh}) -\E_{\theta_k}\left[{F}(\theta_k; z) \right]}^2\overset{(b)}{\leq} 2G_1^4
    \end{align*}
    where $\E_{\theta_k}$ denotes the expectation over the distribution $z = (s,a,s')$ where $(s, a)\sim \nu^{\pi_{\theta_k}}$, $s'\sim P(\cdot|s, a)$. The equality $(a)$ follows from the definition of the Fisher matrix, and $(b)$ is a consequence of assumption \ref{assump:score_func_bounds}. Statements (a) and (c), therefore, directly follow from Lemma \ref{lem:expect_bound_grad}. 

    To prove the other statements, recall the definition of $ \hat{\nabla}_\theta J_g(\theta_k, \xi_g^k;\cdot)$ from \eqref{eq_22_washim_new}. Note the following relations for arbitrary $\theta_k, \xi_{g}^k=[\eta_g^k, \zeta_g^k]$.
    \begin{align*}
        \E_{\theta_k}&\left[\hat{\nabla}_\theta J_g(\theta_k, \xi_g^k;z)\right] - \nabla_\theta J_g(\theta_k)\\
        &=\E_{\theta_k}\left[\bigg\{g(s, a) - \eta_{g}^k + \langle \phi_g(s') - \phi_g(s), \zeta_{g}^k \rangle\bigg\}\nabla_{\theta} \log{\pi_{\theta_k}}(a|s) \right] - \nabla_\theta J_g(\theta_k)\\
        &\overset{(a)}{=}\underbrace{\E_{\theta_k}\left[\bigg\{\eta_{g}^{*}(\theta_k) - \eta^k_{g} + \langle \phi_g(s') - \phi_g(s), \zeta^k_{g}-\zeta_{g}^{*}(\theta_k) \rangle\bigg\}\nabla_{\theta} \log{\pi_{\theta_k}}(a|s) \right]}_{T_0}\\
        &+\underbrace{\E_{\theta_k}\left[\bigg\{V_g^{\pi_{\theta_k}}(s) - \langle \phi_g(s), \zeta_{g}^{*}(\theta_k)\rangle + \langle \phi_g(s'), \zeta_{g}^{*}(\theta_k)\rangle - V_g^{\pi_{\theta_k}}(s')\bigg\}\nabla_{\theta} \log{\pi_{\theta_k}}(a|s) \right]}_{T_1}\\
        &+\underbrace{\E_{\theta_k}\left[\bigg\{V_g^{\pi_{\theta_k}}(s') - \eta_{g}^{*}(\theta_k) + g(s, a) - V_g^{\pi_{\theta_k}}(s)\bigg\}\nabla_{\theta} \log{\pi_{\theta_k}}(a|s) \right]-\nabla_{\theta} J_g(\theta_k)}_{T_2}
    \end{align*}
    We will use the notation that $\xi_{g}^{*}(\theta_k) =[\eta_{g}^{*}(\theta_k), \zeta_{g}^{*}(\theta_k)]^{\top}$. Observe that
    \begin{align}
        \norm{T_0}^2 &\overset{(a)}{=} \cO \left(G_1^2\norm{\xi^k_{g}-\xi_g^{*}(\theta_k)}^2\right), ~\norm{T_1}^2\overset{(b)}{=} \cO\left(G_1^2 \epsilon_{\mathrm{app}}\right), ~\text{and}~ T_2\overset{(c)}{=} 0
    \end{align}
    where $(a)$ follows from Assumption \ref{assump:score_func_bounds} and the boundedness of the feature map, $\phi$ while $(b)$ is a consequence of Assumption \ref{assump:score_func_bounds} and \ref{assump:critic-error}. Finally, $(c)$ is an application of Bellman's equation and the fact that $\eta_g^*(\theta_k) = J_g^{\pi_{\theta_k}}$, which can be easily verified. We get,
    \begin{align}
        \label{eq_appndx_washim_46}
        \norm{\E_{\theta_k}\left[\hat{\nabla}_\theta J_g(\theta_k;z)\right] - \nabla_\theta J_g(\theta_k)}^2 \leq \delta_{k,g}^2 = \cO \left(G_1^2\norm{\xi^k_{g}-\xi_g^{*}(\theta_k)}^2 + G_1^2 \epsilon_{\mathrm{app}}\right)
    \end{align}

    Moreover, observe that, for arbitrary $z_{kh}^j\in \mathcal{S}\times \mathcal{A}\times \mathcal{S}$ 
    \begin{align}
    \label{eq_appndx_washim_47}
        \norm{\hat{\nabla}_\theta J_g(\theta_k, \xi_g^k;z^j_{kh}) - \E_{\theta_k}\left[\hat{\nabla}_\theta J_g(\theta_k, \xi_g^k;z)\right]}^2 \overset{(a)}{\leq} \sigma_{k,g}^2 = \cO\left(G_1^2\norm{\xi^k_{g}}^2\right)
    \end{align}
    where $(a)$ results from Assumption \ref{assump:score_func_bounds} and the boundedness of $\phi$. We can, thus, conclude statements (c) and (d) by applying \eqref{eq_appndx_washim_46} and \eqref{eq_appndx_washim_47} in Lemma \ref{lem:expect_bound_grad}. To prove the statement (e), note that
    \begin{align*}
        &\E_{\theta_k}\left[\E_{k}\left[\hat{\nabla}_\theta J_g(\theta_k, \xi_g^k;z)  \right]\right] - \nabla_\theta J_g(\theta_k)\\
        &= \E_{\theta_k}\left[\bigg\{g(s, a) - \E_{k}[\eta_{g}^k] + \langle \phi_g(s') - \phi_g(s), \E_{k}[\zeta_{g}^k] \rangle\bigg\}\nabla_{\theta} \log {\pi_{\theta_k}}(a|s) \right] - \nabla_\theta J_g(\theta_k)\\
        &\overset{(a)}{=}\underbrace{\E_{\theta_k}\left[\bigg\{\eta_g^{*}(\theta_k) - \E_{k}[\eta_{g}^k] + \langle \phi_g(s') - \phi_g(s), \E_{k}[\zeta_{g}^k]-\zeta^{*}_g(\theta_k) \rangle\bigg\}\nabla_{\theta} \log {\pi_{\theta_k}}(a|s) \right]}_{T_0}+\\
        &+ \underbrace{\E_{\theta_k}\left[\bigg\{V_g^{\pi_{\theta_k}}(s) - \langle \phi_g(s), \zeta^{*}_g(\theta_k)\rangle + \langle \phi_g(s'),  \zeta^{*}_g(\theta_k)\rangle - V_g^{\pi_{\theta_k}}(s')\bigg\}\nabla_{\theta} \log {\pi_{\theta_k}}(a|s) \right]}_{T_1}\\
        &+ \underbrace{\E_{\theta_k}\left[\bigg\{V_g^{\pi_{\theta_k}}(s') - \eta_g^{*}(\theta_k) +g(s, a) - V_g^{\pi_{\theta_k}}(s)\bigg\}\nabla_{\theta} \log {\pi_{\theta_k}}(a|s) \right]-\nabla_{\theta} J_g(\theta_k)}_{T_2}
    \end{align*}
    Observe the following bounds.
    \begin{align}
        \norm{T_0}^2 &\overset{(a)}{=} \cO \left(G_1^2\norm{\E_{k}[\xi_{g}^k]-\xi_g^{*}(\theta_k)}^2\right), ~\norm{T_1}^2\overset{(b)}{=} \cO\left(G_1^2 \epsilon_{\mathrm{app}}\right), ~\text{and}~ T_2\overset{(c)}{=} 0
    \end{align}
    where $(a)$ follows from Assumption \ref{assump:score_func_bounds} and the boundedness of the feature map, $\phi$ while $(b)$ is a consequence of Assumption \ref{assump:score_func_bounds} and \ref{assump:critic-error}. Finally, $(c)$ is an application of Bellman's equation. We get,
    \begin{align}
    \label{eq_appndx_washim_46_}
        \norm{\E_{\theta_k}\left[\E_{k}\left[\hat{\nabla}_\theta J_g(\theta_k, \xi_g^k;z)\right]\right] - \nabla_\theta J_g(\theta_k)}^2 \leq \bar{\delta}_{k,g}^2 = \cO \left(G_1^2\norm{\E_{k}[\xi_{g}^k]-\xi_g^{*}(\theta_k)}^2 + G_1^2 \epsilon_{\mathrm{app}}\right)
    \end{align}
    Using the above bound, we deduce the following.
    \begin{align*}
        &\norm{\E_{k}\left[\nabla_\theta J^{\mathrm{MLMC}}_{g,kh}\right]-\nabla_{\theta} J_g(\theta_k)}^2\\
        &\leq 2\norm{\E_{k}\left[\nabla^{\mathrm{MLMC}}_\theta J_{g,kh}\right] - \E_{\theta_k}\left[\E_{k}\left[\hat{\nabla}_\theta J_g(\theta_k, \xi_g^k;z)\right]\right]}^2 \\
        &\hspace{1cm}+ 2\norm{\E_{\theta_k}\left[\E_{k}\left[\hat{\nabla}_\theta J_g(\theta_k, \xi_g^k;z)\right]\right] - \nabla_\theta J_g(\theta_k)}^2\\
        &\overset{(a)}{\leq} 2\E_{k}\norm{\E_{k, h}\left[\nabla^{\mathrm{MLMC}}_\theta J_{g,kh}\right] - \E_{\theta_k}\left[\hat{\nabla}_\theta J_g(\theta_k, \xi_g^k;z)\right]}^2 + \cO\left(\bar{\delta}_{k,g}^2\right)\\
        &\overset{(b)}{\leq} \cO\left(\tau_{\mathrm{mix}}T_{\max}^{-1}\bar{\sigma}_{k,g}^2 + \bar{\delta}_{k,g}^2\right)
    \end{align*}
    where $(a)$ follows from \eqref{eq_appndx_washim_46_}, $(b)$ follows from Lemma \ref{lem:expect_bound_grad}(a), \ref{lem:tech_markov}, and the definition of $\bar{\sigma}^2_{k,g}$.
\end{proof}

Combining Lemma \ref{lemma_aux_2} with Assumptions \ref{assump:score_func_bounds} and \ref{assump:FND_policy}, we obtain the following for a given policy parameter $\theta$ and $\forall g \in \{r,c\}$,
\begin{align}
    \mu\leq \norm{F(\theta)}\leq G_1^2, ~\text{and}~ \norm{\nabla_\theta J_g(\theta)}\leq \cO\left(G_1\tau_{\mathrm{mix}}\right)
\end{align}

Combining the above results with Lemma \ref{lem:F-J-bias} and invoking Theorem \ref{thm:generalrecursion},  we then obtain given that $T_{\max} \geq \frac{8G_1^4 \tau_{\mathrm{mix}}}{\mu}$ and $\gamma_\omega \leq \frac{\mu}{4(6 G_1^4 \tau_{\mathrm{mix}} \log T_{\max} + 2 G_1^2  \tau_{\mathrm{mix}}^2 \log T_{\max})}$.
\begin{align*}
    \E_k\left[\norm{\omega_{g}^k-\omega_{g, k}^{*}}^2\right]&\leq\exp\left(-{H\gamma_\omega \mu}\right)\norm{\omega_0-\omega_{g, k}^{*}}^2 + \Tilde{\cO}\left(\gamma_\omega R_0+R_1\right), \\
    \norm{\E_k[\omega_{g}^k]-\omega_{g, k}^{*}}^2&\leq\exp\left(-{H\gamma_\omega \mu}\right)\norm{\omega_0 -\omega_{g, k}^{*}}^2 \\
    &+\dfrac{\mu\gamma_\omega+1}{\mu^2}\left[\mathcal{O}(T_{\max}^{-1}G_{1}^4\tau_{\mathrm{mix}})\bigg\{\norm{\omega_0 -\omega_{g, k}^{*}}^2+\mathcal{O}\left(\gamma_\omega R_0+R_1\right)\bigg\} + \bar{R}_1 \right]
\end{align*}
where the terms $R_0, R_1, \bar{R}_1$ are defined as follows.
\begin{align*}
    R_0&= \tilde{\cO}\left(\mu^{-3}G_1^6\tau_{\mathrm{mix}}^3+ \mu^{-1}G_1^2\tau_{\mathrm{mix}}\E_k\left[\norm{\xi_{g}^k}^2\right] + \mu^{-1}G_1^2\E_k\left[\norm{\xi_{g}^k-\xi_g^{*}(\theta_k)}^2\right] + \mu^{-1}G_1^2\epsilon_{\mathrm{app}}\right),\\    R_1&=\mu^{-2}G_1^2\cO\left(T_{\max}^{-1}\mu^{-2}G_1^4\tau^3_{\mathrm{mix}} + T_{\max}^{-1}\tau_{\mathrm{mix}}\E_k\left[\norm{\xi_{g}^k}^2\right]+ \E_k\left[\norm{\xi_{g}^k-\xi_g^{*}(\theta_k)}^2\right] + \epsilon_{\mathrm{app}}\right),\\
    \bar{R}_1&=\cO\left(T_{\max}^{-1}\mu^{-2}G_1^6\tau^3_{\mathrm{mix}} + T_{\max}^{-1}G_1^2\tau_{\mathrm{mix}}\E_k\left[\norm{\xi_{g}^k}^2\right] + G_1^2\norm{\E_k[\xi_{g}^k]-\xi_g^{*}(\theta_k)}^2 + G_1^2\epsilon_{\mathrm{app}}\right)
\end{align*}

Moreover, note that
\begin{align*}
    \E_k\left[\norm{\xi_{g}^k}^2\right]\leq 2\E_k\left[\norm{\xi_{g}^k-\xi_g^{*}(\theta_k)}^2\right] + 2\E_k\left[\norm{\xi_g^{*}(\theta_k)}^2\right] \overset{(a)}{\leq} \cO\left(\E_k\left[\norm{\xi_{g}^k-\xi_g^{*}(\theta_k)}^2\right]+ \lambda^{-2}c_{\gamma}^2\right)
\end{align*}
where $(a)$ uses \eqref{eq_imp_appndx_58} for sufficiently large $c_{\gamma}$ and the definition that $\xi^{*}_{g}(\theta_k) = [\mathbf{A}_{g}(\theta_k)]^{-1}\mathbf{b}_{g}(\theta_k)$. If we set $\gamma_\omega = \frac{2 \log T}{\mu H}\leq\frac{\mu}{4(6 G_1^4 \tau_{\mathrm{mix}} \log T_{\max} + 2 G_1^2  \tau_{\mathrm{mix}}^2 \log T_{\max})}$, we get
\begin{align*}
    &\E_k\left[\norm{\omega_{g}^k-\omega_{g, k}^{*}}^2\right]\leq  \tilde{\cO}\left(\left(\frac{1}{H}+\frac{1}{T_{\max}}\right)\left\{\mu^{-4}G_1^6\tau_{\mathrm{mix}}^3+ \mu^{-2}\lambda^{-2}G_1^2c_{\gamma}^2 \tau_{\mathrm{mix}}\right\}\right)\\
    & \hspace{1cm}+ \dfrac{1}{T^2}\norm{\omega_0-\omega_{g, k}^{*}}^2 + \Tilde{\cO}\left( \mu^{-2}G_1^2 \bigg\{ \E_k\left[\norm{\xi_{g}^k-\xi_g^{*}(\theta_k)}^2\right] + \epsilon_{\mathrm{app}}  \bigg\}\right), \\
    &\norm{\E_k[\omega_g^k]-\omega_{g, k}^*}^2\leq \mu^{-2}G_1^2\Tilde{\cO}\left(\norm{\E_k[\xi_{g}^k]-\xi_g^{*}(\theta_k)}^2 + \epsilon_{\mathrm{app}}\right)+\left(\dfrac{1}{T^2}+\dfrac{G_1^4\tau_{\mathrm{mix}}}{\mu^2 T_{\max}}\right)\norm{\omega_0-\omega_{g,k}^{*}}^2 \\
    &\hspace{1cm}+\Tilde{\cO}\left( \dfrac{G_1^4\tau_{\mathrm{mix}}}{\mu^2 T_{\max}} \bigg\{ \mu^{-2}G_1^2\tau_{\mathrm{mix}}\E_k\left[\norm{\xi_{g}^k-\xi_g^{*}(\theta_k)}^2\right] + \mu^{-4}G_1^6\tau_{\mathrm{mix}}^3 + \mu^{-2}\lambda^{-2}G_1^2 c_{\gamma}^2 \tau_{\mathrm{mix}} \bigg\}\right)
\end{align*}
Substituting $\omega_0=0$, $\|\omega_{g, k}^*\| = \|[F(\theta_k)]^{-1}\nabla_{\theta} J_g(\theta_k)\| = \mathcal{O}(\mu^{-1}G_1\tau_{\mathrm{mix}})$ (follows from Lemma \ref{lemma_aux_2}, and assumptions \ref{assump:score_func_bounds}and \ref{assump:FND_policy}), we can simplify the above bounds as follows.
\begin{align}
    \E_k\left[\norm{\omega_{g}^k-\omega_{g, k}^{*}}^2\right]&\leq  \tilde{\cO}\left(\left(\frac{1}{H}+\frac{1}{T_{\max}}\right)\mu^{-4}\lambda^{-2}G_1^6c_{\gamma}^2 \tau^3_{\mathrm{mix}}\right) + \dfrac{G_1^2\tau_{\mathrm{mix}}^2}{\mu^2T^2} \nonumber\\
    \label{eq_appndx_66_washim}
    & + \Tilde{\cO}\left( \frac{G_1^2}{\mu^2} \bigg\{ \E_k\left[\norm{\xi_{g}^k-\xi_g^{*}(\theta_k)}^2\right] + \epsilon_{\mathrm{app}}  \bigg\}\right), \\
    \norm{\E_k[\omega_g^k]-\omega_{g, k}^*}^2 &\leq \frac{G_1^2}{\mu^2}\Tilde{\cO}\left(\norm{\E_k[\xi_{g}^k]-\xi_g^{*}(\theta_k)}^2 + \epsilon_{\mathrm{app}}\right)+\left(\dfrac{1}{T^2}+\dfrac{G_1^4\tau_{\mathrm{mix}}}{\mu^2 T_{\max}}\right)\frac{G_1^2\tau_{\mathrm{mix}}^2}{\mu^2}\nonumber\\
    \label{eq_appndx_67_washim}
    &+\Tilde{\cO}\left( \dfrac{G_1^6\tau^2_{\mathrm{mix}}}{\mu^4 T_{\max}} \bigg\{ \E_k\left[\norm{\xi_{g}^k-\xi_g^{*}(\theta_k)}^2\right] + \mu^{-2}\lambda^{-2}G_1^4 c_{\gamma}^2\tau_{\mathrm{mix}}^2 \bigg\}\right)
\end{align}
This completes the proof of Theorem \ref{thm:npg-main}.


\section{Proof of Main Theorems (Theorem \ref{thm_global_convergence_2} and Theorem \ref{thm_global_convergence_1})}

\subsubsection{Rate of Convergence of the Objective}

Combining \eqref{eq_appndx_56_washim}, \eqref{eq_appndx_washim_57} and \eqref{eq_appndx_66_washim} and \eqref{eq_appndx_67_washim}, we obtain the following results.
\begin{align}
    \E_k\left[\norm{\omega_{g}^k-\omega_{g, k}^{*}}^2\right]&\leq  \tilde{\cO}\left(\frac{G_1^6c_{\gamma}^2 \tau^3_{\mathrm{mix}}}{\mu^{4}\lambda^{2}} + \dfrac{G_1^2\tau_{\mathrm{mix}}^2}{\mu^2T^2} \right) + \Tilde{\cO}\left( \frac{G_1^2}{\mu^2} \bigg\{ \frac{c_{\gamma}^2}{\lambda^2 T^2} + \frac{c_{\gamma}^4 \tau_{\mathrm{mix}}}{\lambda^4}+ \epsilon_{\mathrm{app}}  \bigg\}\right)\nonumber\\
    \label{eq_appndx_68_washim}
    &=\tilde{\cO}\left(\frac{G_1^2}{\mu^2}\epsilon_{\mathrm{app}}+\frac{G_1^6c_{\gamma}^4 \tau^3_{\mathrm{mix}}}{\mu^{4}\lambda^{4}}\right), \\
    \norm{\E_k[\omega_g^k]-\omega_{g, k}^*}^2 &\leq \frac{G_1^2}{\mu^2}\Tilde{\cO}\left(\frac{c_{\gamma}^2}{\lambda^2 T^2} + \frac{c_{\gamma}^6 \tau_{\mathrm{mix}}^2}{\lambda^6 T_{\max}} + \epsilon_{\mathrm{app}}\right)+\left(\dfrac{1}{T^2}+\dfrac{G_1^4\tau_{\mathrm{mix}}}{\mu^2 T_{\max}}\right)\frac{G_1^2\tau_{\mathrm{mix}}^2}{\mu^2}\nonumber\\
    &+\Tilde{\cO}\left( \dfrac{G_1^6\tau^2_{\mathrm{mix}}}{\mu^4 T_{\max}} \bigg\{ \frac{c_{\gamma}^2}{\lambda^2T^2} + \frac{c_{\gamma}^4 \tau_{\mathrm{mix}}}{\lambda^4}+ \frac{G_1^4 c_{\gamma}^2\tau_{\mathrm{mix}}^2}{\mu^{2}\lambda^{2}} \bigg\}\right)\nonumber\\
    & = \tilde{\cO}\left(\frac{G_1^2}{\mu^2}\epsilon_{\mathrm{app}}+ \frac{G_1^2c_{\gamma}^2\tau^2_{\mathrm{mix}}}{\mu^2\lambda^2 T^2} + \frac{G_1^{10}c^6_{\gamma}\tau^4_{\mathrm{mix}}}{\mu^6\lambda^6 T_{\max}}\right)
    \label{eq_appndx_69_washim}
\end{align}

We can decompose $\E\Vert(\E_k\left[\omega_k\right]-\omega^*_k)\Vert$ and $\E\Vert \omega_k -\omega_k^*\Vert^2$ using the definition that $\omega_k = \omega_r^k + \lambda_k \omega_c^k$. Using \eqref{eq_appndx_68_washim} and \eqref{eq_appndx_69_washim}, and the fact that $\lambda_k \in [0, 2/\delta]$, we obtain
\begin{align}
    \E\Vert(\E_k\left[\omega_k\right]-\omega^*_k)\Vert &\leq  \left(1+\frac{2}{\delta}\right)\tilde{\cO}\left(\frac{G_1}{\mu}\sqrt{\epsilon_{\mathrm{app}}}+ \frac{G_1c_{\gamma}\tau_{\mathrm{mix}}}{\mu\lambda T} + \frac{G_1^{5}c^3_{\gamma}\tau^2_{\mathrm{mix}}}{\mu^3\lambda^3 \sqrt{T_{\max}}}\right)\label{eq:omega_condition_bound}\\
    \E\Vert \omega_k -\omega_k^*\Vert^2 &\leq \left(1+\dfrac{4}{\delta^2}\right)\tilde{\cO}\left(\frac{G_1^2}{\mu^2}\epsilon_{\mathrm{app}}+\frac{G_1^6c_{\gamma}^4 \tau^3_{\mathrm{mix}}}{\mu^{4}\lambda^{4}}\right)\label{eq:omega_uncondition_bound}
\end{align}

Setting $T_{\max}=T$ in \eqref{eq:omega_condition_bound} and \eqref{eq:omega_uncondition_bound}, and using Lemma \ref{lemma_grad_JL_bound} and \eqref{eq:omegadeconposition} in \eqref{eq:general_bound}   would obtain
\begin{align*}
    &\frac{1}{K}\E\sum_{k=0}^{K-1}\bigg(\mathcal{L}(\pi^*, \lambda_k)-\mathcal{L}(\theta_k,\lambda_k)\bigg) \leq \sqrt{\epsilon_{\mathrm{bias}}} + \frac{1}{\alpha K}\E_{s\sim d^{\pi^*}}[KL(\pi^*(\cdot\vert s)\Vert\pi_{\theta_0}(\cdot\vert s))] \nonumber\\
    &+\left(1+\frac{2}{\delta}\right)\tilde{\cO}\left(\frac{G_1^2}{\mu}\sqrt{\epsilon_{\mathrm{app}}}+  \frac{G_1^{6}c^3_{\gamma}\tau^2_{\mathrm{mix}}}{\mu^3\lambda^3 \sqrt{T}}\right) + \alpha G_2\left(1+\dfrac{4}{\delta^2}\right)\tilde{\cO}\left(\frac{G_1^2}{\mu^2}\epsilon_{\mathrm{app}}+\frac{G_1^6c_{\gamma}^4 \tau^3_{\mathrm{mix}}}{\mu^{4}\lambda^{4}}\right)\nonumber\\
    & + \alpha\cO\left( \left(1+\frac{4}{\delta^2}\right)\frac{G_1^2 \tau_{\mathrm{mix}}^2}{\mu^2}\right) 
\end{align*}

Using the definition of the Lagrange function, the above inequality can be equivalently written as
\begin{align}
    &\frac{1}{K}\E\sum_{k=0}^{K-1}\bigg(J_r^{\pi^*}-J_r(\theta_k)\bigg) \leq \sqrt{\epsilon_{\mathrm{bias}}} + \frac{1}{\alpha K}\E_{s\sim d^{\pi^*}}[KL(\pi^*(\cdot\vert s)\Vert\pi_{\theta_0}(\cdot\vert s))] \nonumber\\
    &+\left(1+\frac{2}{\delta}\right)\tilde{\cO}\left(\frac{G_1^2}{\mu}\sqrt{\epsilon_{\mathrm{app}}}+  \frac{G_1^{6}c^3_{\gamma}\tau^2_{\mathrm{mix}}}{\mu^3\lambda^3 \sqrt{T}}\right) + \alpha G_2\left(1+\dfrac{4}{\delta^2}\right)\tilde{\cO}\left(\frac{G_1^2}{\mu^2}\epsilon_{\mathrm{app}}+\frac{G_1^6c_{\gamma}^4 \tau^3_{\mathrm{mix}}}{\mu^{4}\lambda^{4}}\right)\nonumber\\
    & + \alpha\cO\left( \left(1+\frac{4}{\delta^2}\right)\frac{G_1^2 \tau_{\mathrm{mix}}^2}{\mu^2}\right) -\frac{1}{K}\sum_{k=0}^{K-1}\E\left[\lambda_k\left(J_c^{\pi^*}-J_c(\theta_k)\right)\right]
    \label{eq_appndx_72_washim}
\end{align}

We will now extract the objective convergence rate from the convergence rate of the Lagrange function stated above. Note that,
\begin{equation}
    \label{eq:bound_lambdak}
    \begin{aligned}
        0\leq (\lambda_{K})^2&\leq\sum_{k=0}^{K-1}\bigg((\lambda_{k+1})^2-(\lambda_{k})^2\bigg)\\
        &=\sum_{k=0}^{K-1}\bigg(\mathcal{P}_{[0,\frac{2}{\delta}]}\big[\lambda_{k}-\beta\eta_{c}^k\big]^2-(\lambda_{k})^2\bigg)\\
        &\leq\sum_{k=0}^{K-1}\bigg(\big[\lambda_{k}-\beta \eta_{c}^k\big]^2-(\lambda_{k})^2\bigg)\\
        &=-2\beta\sum_{k=0}^{K-1}\lambda_{k}\eta_{c}^k+\beta^2\sum_{k=0}^{K-1}(\eta_{c}^{k})^2\\
        &\overset{(a)}\leq 2\beta\sum_{k=0}^{K-1}\lambda_{k}(J_{c}^{\pi^*}- \eta_{c}^k)+2\beta^2\sum_{k=0}^{K-1}(\eta_{c}^k)^2\\
        &= 2\beta\sum_{k=0}^{K-1}\lambda_{k}(J_{c}^{\pi^*}- J_{c}(\theta_k))+2\beta\sum_{k=0}^{K-1}\lambda_{k}(J_{c}(\theta_k)- \eta_{c}^k)+2\beta^2\sum_{k=0}^{K-1}(\eta_{c}^k)^2
    \end{aligned}
\end{equation}

Inequality (a) holds because $\theta^*$ is a feasible solution to the constrained optimization problem. Rearranging items and taking the expectation, we have

\begin{equation}
    -\frac{1}{K}\sum_{k=0}^{K-1}\E\bigg[\lambda_{k}(J_{c}^{\pi^*}- J_{c}(\theta_k))\bigg]\leq \frac{1}{K}\sum_{k=0}^{K-1}\E\bigg[\lambda_{k}(J_{c}(\theta_k)- \eta_{c}^k)\bigg]+\frac{\beta}{K}\sum_{k=0}^{K-1}\E[\eta_{c}^k]^2
\end{equation}

Note that, unlike the discounted reward case, the average reward estimate $\eta_{c}^k$ is no longer unbiased. However, by Theorem \ref{th:acc_td} and the facts that $|c(s,a)|\leq 1,\forall(s,a)\in\mathcal{S}\times\mathcal{A}$, and $\eta_c^*(\theta_k) = J_c(\theta_k)$, we get the following inequality.
\begin{equation}
    \begin{aligned}
        -\frac{1}{K}\sum_{k=0}^{K-1}\E\bigg[\lambda_k(J_{c}^{\pi^*}- J_c(\theta_k))\bigg]&\leq \frac{1}{K}\sum_{k=0}^{K-1}\E\bigg[\lambda_{k}\bigg(\eta^*_{c}(\theta_k)- \E_k[\eta_{c}^k]\bigg)\bigg] + \beta\\
        &\leq \frac{1}{ K}\sum_{k=0}^{K-1}\E\bigg[|\lambda_k|\bigg|\eta^*_{c}(\theta_k)- \E_k[\eta_{c}^k]\bigg|\bigg] + \beta\\
        &\leq \tilde{\cO}\left(\frac{c_{\gamma}^3\tau_{\mathrm{mix}}}{\lambda^3\delta \sqrt{T}} + \beta\right)
    \end{aligned}
\end{equation}
where the last inequality utilizes \eqref{eq_appndx_washim_57}, the fact that $\lambda_k\in [0, 2/\delta]$, and $T_{\max}=T$. Combining with \eqref{eq_appndx_72_washim}, we arrive at the following result.
\begin{align}
    &\frac{1}{K}\E\sum_{k=0}^{K-1}\bigg(J_r^{\pi^*}-J_r(\theta_k)\bigg) \leq \sqrt{\epsilon_{\mathrm{bias}}} + \frac{1}{\alpha K}\E_{s\sim d^{\pi^*}}[KL(\pi^*(\cdot\vert s)\Vert\pi_{\theta_0}(\cdot\vert s))] \nonumber\\
    &+\left(1+\frac{2}{\delta}\right)\tilde{\cO}\left(\frac{G_1^2}{\mu}\sqrt{\epsilon_{\mathrm{app}}}+  \frac{G_1^{6}c^3_{\gamma}\tau^2_{\mathrm{mix}}}{\mu^3\lambda^3 \sqrt{T}}\right) + \alpha G_2\left(1+\dfrac{4}{\delta^2}\right)\tilde{\cO}\left(\frac{G_1^2}{\mu^2}\epsilon_{\mathrm{app}}+\frac{G_1^6c_{\gamma}^4 \tau^3_{\mathrm{mix}}}{\mu^{4}\lambda^{4}}\right)\nonumber\\
    & + \alpha\cO\left( \left(1+\frac{4}{\delta^2}\right)\frac{G_1^2 \tau_{\mathrm{mix}}^2}{\mu^2}\right) + \tilde{\cO}\left(\frac{c_{\gamma}^3\tau_{\mathrm{mix}}}{\lambda^3\delta \sqrt{T}} + \beta\right) \nonumber\\
    & \leq \tilde{\cO} \left(\sqrt{\epsilon_{\mathrm{bias}}}+ \sqrt{\epsilon_{\mathrm{app}}} + \alpha + \beta + \frac{\tau^2_{\mathrm{mix}}}{\sqrt{T}} + \frac{1}{\alpha K}\right)
    \label{eq:bound_Jr_final}
\end{align}

\subsubsection{Rate of Constraint Violation}
Given the dual update in algorithm \ref{alg:pdranac}, we have
\begin{equation}
    \begin{aligned}
        \left\vert\lambda_{k+1} - \frac{2}{\delta}\right\vert^2 &\overset{(a)}{\leq} \bigg|\lambda_{k} - \beta J_c(\theta_{k})   -\frac{2}{\delta}\bigg|^2\\
        &=\left|\lambda_{k} -\frac{2}{\delta}\right|^2 -2\beta J_c(\theta_k)\left(\lambda_{k}  -\frac{2}{\delta}\right) +\beta^2 J_c(\theta_{k})^2\\
        &\leq\left|\lambda_{k} - \frac{2}{\delta}\right|^2 -2\beta J_c(\theta_{k})\left(\lambda_{k}  -\frac{2}{\delta}\right)  + \beta^2
    \end{aligned}
\end{equation}
where $(a)$ is because of the non-expansiveness of projection $\mathcal{P}_\Lambda$. Averaging the above inequality over $k=0,\ldots,K-1$ yields
\begin{equation}
    \begin{aligned}
        0\leq\frac{1}{K}\left\vert\lambda_{K} - \frac{2}{\delta}\right\vert^2 
        \leq\frac{1}{K}\left\vert{{\lambda_0  -\frac{2}{\delta}}}\right\vert^2 -\frac{2\beta}{K}\sum_{k=0}^{K-1} J_c(\theta_k)\left(\lambda_{k}  -\frac{2}{\delta}\right)  + \beta^2
    \end{aligned}
\end{equation}

Taking expectations on both sides,
\begin{equation}
    \label{eq.removelamda}
    \frac{1}{K}\sum_{k=0}^{K-1} \E\bigg[J_c(\theta_k)\left(\lambda_k  -\frac{2}{\delta}\right) \bigg] \leq \frac{1}{2\beta K}\left\vert{{\lambda_0 -\frac{2}{\delta}}}\right\vert^2 + \frac{\beta}{2} \leq \frac{2}{\delta^2 \beta K} + \frac{\beta}{2}
\end{equation}	
where the last inequality utilizes $\lambda_0=0$. Notice that $\lambda_k J_c^{\pi^*}\geq 0$, $\forall k$. Using the above inequality to \eqref{eq_appndx_72_washim}, we therefore have,
\begin{align}
    &\frac{1}{K}\sum_{k=0}^{K-1}\E\bigg(J_r^{\pi^*}-J_r(\theta_k)\bigg)+ \frac{2}{\delta}\sum_{k=0}^{K-1}\frac{1}{K}\E \bigg[-J_c(\theta_k)\bigg]\leq \sqrt{\epsilon_{\mathrm{bias}}}+ \frac{2}{\delta^2\beta K} + \frac{\beta}{2}\nonumber\\
    &+\left(1+\frac{2}{\delta}\right)\tilde{\cO}\left(\frac{G_1^2}{\mu}\sqrt{\epsilon_{\mathrm{app}}}+  \frac{G_1^{6}c^3_{\gamma}\tau^2_{\mathrm{mix}}}{\mu^3\lambda^3 \sqrt{T}}\right) + \alpha G_2\left(1+\dfrac{4}{\delta^2}\right)\tilde{\cO}\left(\frac{G_1^2}{\mu^2}\epsilon_{\mathrm{app}}+\frac{G_1^6c_{\gamma}^4 \tau^3_{\mathrm{mix}}}{\mu^{4}\lambda^{4}}\right)\nonumber\\
    & + \alpha\cO\left( \left(1+\frac{4}{\delta^2}\right)\frac{G_1^2 \tau_{\mathrm{mix}}^2}{\mu^2}\right) + \frac{1}{\alpha K}\E_{s\sim d^{\pi^*}}[KL(\pi^*(\cdot\vert s)\Vert\pi_{\theta_0}(\cdot\vert s))]
    \label{eq:bound_Jc1}
\end{align}

We define a new policy $\bar{\pi}$ which uniformly chooses the policy $\pi_{\theta_k}$ for $k\in[K]$. By the occupancy measure method, $J_g(\theta_k)$ is linear in terms of an occupancy measure induced by policy $\pi_{\theta_k}$. Thus, 
\begin{equation}\label{eq_avg_value}
    \frac{1}{K}\sum_{k=0}^{K-1}J_g(\theta_k)=J_g^{\bar\pi}, ~g\in \{r, c\}
\end{equation}
Injecting the above relation to \eqref{eq:bound_Jc1}, we have
\begin{equation}
    \begin{aligned}
        &\E\bigg[J_r^{\pi^*}-J_r^{\bar\pi}\bigg]+\frac{2}{\delta}\E\bigg[-J_c^{\bar\pi}\bigg]\leq \sqrt{\epsilon_{\mathrm{bias}}}+\frac{2}{\delta^2\beta K} + \dfrac{\beta}{2} \\
        &+\left(1+\frac{2}{\delta}\right)\tilde{\cO}\left(\frac{G_1^2}{\mu}\sqrt{\epsilon_{\mathrm{app}}}+  \frac{G_1^{6}c^3_{\gamma}\tau^2_{\mathrm{mix}}}{\mu^3\lambda^3 \sqrt{T}}\right) + \alpha G_2\left(1+\dfrac{4}{\delta^2}\right)\tilde{\cO}\left(\frac{G_1^2}{\mu^2}\epsilon_{\mathrm{app}}+\frac{G_1^6c_{\gamma}^4 \tau^3_{\mathrm{mix}}}{\mu^{4}\lambda^{4}}\right)\nonumber\\
        & + \alpha\cO\left( \left(1+\frac{4}{\delta^2}\right)\frac{G_1^2 \tau_{\mathrm{mix}}^2}{\mu^2}\right) + \frac{1}{\alpha K}\E_{s\sim d^{\pi^*}}[KL(\pi^*(\cdot\vert s)\Vert\pi_{\theta_0}(\cdot\vert s))]
    \end{aligned}
\end{equation}
By Lemma \ref{lem.constraint}, we arrive at,
\begin{align}
\label{eq:vio}
    &\E\bigg[-J_c^{\bar\pi}\bigg]\leq  \delta\sqrt{\epsilon_{\mathrm{bias}}}+\frac{1}{\delta\beta K} + \frac{\delta\beta}{2} \nonumber \\
    &+\left(1+\frac{\delta}{2}\right)\tilde{\cO}\left(\frac{G_1^2}{\mu}\sqrt{\epsilon_{\mathrm{app}}}+  \frac{G_1^{6}c^3_{\gamma}\tau^2_{\mathrm{mix}}}{\mu^3\lambda^3 \sqrt{T}}\right) + \alpha G_2\left(\frac{\delta}{2}+\dfrac{2}{\delta}\right)\tilde{\cO}\left(\frac{G_1^2}{\mu^2}\epsilon_{\mathrm{app}}+\frac{G_1^6c_{\gamma}^4 \tau^3_{\mathrm{mix}}}{\mu^{4}\lambda^{4}}\right)\nonumber\\
    & + \alpha\cO\left( \left(\frac{\delta}{2}+\frac{2}{\delta}\right)\frac{G_1^2 \tau_{\mathrm{mix}}^2}{\mu^2}\right) + \frac{\delta}{2\alpha K}\E_{s\sim d^{\pi^*}}[KL(\pi^*(\cdot\vert s)\Vert\pi_{\theta_0}(\cdot\vert s))]\nonumber\\
    & = \tilde{\cO}\left(\sqrt{\epsilon_{\mathrm{bias}}} + \sqrt{\epsilon_{\mathrm{app}}} + \alpha + \beta + \frac{1}{\alpha K} + \frac{1}{\beta K} + \frac{\tau_{\mathrm{mix}}^2}{\sqrt{T}}\right)
\end{align}

\subsubsection{Optimal Choice of $\alpha$, $\beta$, $K$, and $H$}

If $\tau_{\mathrm{mix}}$ is unknown, we can take $\alpha=T^{-a}$, $\beta=T^{-b}$ for some $a, b\in(0,1)$, and $H=T^{\epsilon}$,  $K=T^{1-\epsilon}$ for $ \epsilon \in (0,1)$ then following \eqref{eq:bound_Jr_final} and \eqref{eq:vio}, we can write,
\begin{align*}
    &\frac{1}{K}\sum_{k=0}^{K-1}\E\bigg(J_r^{\pi^*}-J_r(\theta_k)\bigg)\leq  \tilde{\mathcal{O}}\left(\sqrt{\epsilon_{\mathrm{bias}}} + \sqrt{\epsilon_{\mathrm{app}}} + T^{-a}+T^{-b}+T^{-0.5} + T^{-1+\epsilon+a} \right),\\
    &\E\bigg[\frac{1}{K}\sum_{k=0}^{K-1}-J_c(\theta_k)\bigg]\leq \tilde{\mathcal{O}}\left(\sqrt{\epsilon_{\mathrm{bias}}}  + \sqrt{\epsilon_{\mathrm{app}}} + T^{-a}+T^{-b}+T^{-0.5} + T^{-1+\epsilon+a} + T^{-1+\epsilon+b}\right)
\end{align*}
Clearly, the optimal values of $a$ and $b$ can be obtained by solving the following optimization.
\begin{align}
    {\max}_{(a, b)\in (0,1)^2} \min \left\lbrace a, b, 1-\epsilon-a, 1-\epsilon-b \right\rbrace 
\end{align}

By choosing $(a, b) = \left( 1/2, 1/2 \right)$, this would obtain the solution of the above optimization problem. Thus, the convergence rate and constraint violation would both become:
\begin{align}
\label{eq_appndx_84_washim}
    \frac{1}{K}\sum_{k=0}^{K-1}\E\bigg(J_r^{\pi^*}-J_r(\theta_k)\bigg)&\leq  \tilde{\mathcal{O}}\left(\sqrt{\epsilon_{\mathrm{bias}}}  + \sqrt{\epsilon_{\mathrm{app}}} + \frac{1}{T^{(0.5-\epsilon)}}\right)\\
    \label{eq_appndx_85_washim}
    \E\bigg[\frac{1}{K}\sum_{k=0}^{K-1}-J_c(\theta_k)\bigg] &\leq     \tilde{\mathcal{O}}\left(\sqrt{\epsilon_{\mathrm{bias}}}  + \sqrt{\epsilon_{\mathrm{app}}} + \frac{1}{T^{(0.5-\epsilon)}}\right)
\end{align}

 Recall that the above result holds when the conditions for Theorem \ref{th:acc_td} and Theorem \ref{thm:npg-main} are satisfied.
 \begin{align} 
     &\frac{2\log T}{H} \leq \frac{\lambda^2}{24 c_\gamma^2\tau_{\mathrm{mix}} \log T_{\max}} = \Theta\left(\frac{1}{\tau_{\mathrm{mix}}\log T}\right) \label{eq:H_condition1}\\
     &\frac{2\log T}{H} \leq \frac{\mu^2}{4(6 G_1^4 \tau_{\mathrm{mix}} \log T_{\max} + 2 G_1^2  \tau_{\mathrm{mix}}^2 \log T_{\max})} = \Theta\left(\frac{1}{\tau_{\mathrm{mix}}^2\log T}\right) \label{eq:H_condition2}
 \end{align}

In light of the above conditions, \eqref{eq_appndx_84_washim} and \eqref{eq_appndx_85_washim} hold if $H=T^{\epsilon}\geq \tilde{\Theta}(\tau_{\mathrm{mix}}^2)$. If $\tau_{\mathrm{mix}}$ is known, we can set $H=\tilde{\Theta}(\tau_{\mathrm{mix}}^2)$ that satisfies \eqref{eq:H_condition1} and \eqref{eq:H_condition2}. Moreover, we can take $K=\cO(T)$. Thus, by following the same analysis as above, we can obtain:
\begin{align}
    \frac{1}{K}\sum_{k=0}^{K-1}\E\bigg(J_r^{\pi^*}-J_r(\theta_k)\bigg)&\leq  \tilde{\mathcal{O}}\left(\sqrt{\epsilon_{\mathrm{bias}}}  + \sqrt{\epsilon_{\mathrm{app}}} + \frac{1}{\sqrt{T}}\right)\\ \E\bigg[\frac{1}{K}\sum_{k=0}^{K-1}-J_c(\theta_k)\bigg] &\leq \tilde{\mathcal{O}}\left(\sqrt{\epsilon_{\mathrm{bias}}}  + \sqrt{\epsilon_{\mathrm{app}}} + \frac{1}{\sqrt{T}}\right)
\end{align}

This concludes the proof of Theorem \ref{thm_global_convergence_2} and Theorem \ref{thm_global_convergence_1}.


\section{Some Auxiliary Lemmas for the Proofs}

\begin{lemma}
    \label{lemma_aux_2}
    \citep[Lemma 14]{wei2020model} For any ergodic MDP with mixing time $\tau_{\mathrm{mix}}$, the following holds $\forall (s, a)\in\mathcal{S}\times \mathcal{A}$, any policy $\pi$ and $\forall g\in\{r, c\}$.
    \begin{align*}
        (a) |V_g^{\pi}(s)|\leq 5 \tau_{\mathrm{mix}},~~
        (b) |Q_g^{\pi}(s, a)|\leq 6 \tau _{\mathrm{mix}},~~
        (c) |A_g^\pi(s,a)| = \mathcal{O}(\tau_{\mathrm{mix}})
    \end{align*}
\end{lemma}

\begin{lemma}
    \label{lemma_grad_JL_bound}
    For any ergodic MDP with mixing time $\tau_{\mathrm{mix}}$, the following holds $\forall (s, a)\in\mathcal{S}\times \mathcal{A}$ for any policy $\pi_{\theta_k}$ with $
    \theta_k$ satisfying Assumption \ref{assump:score_func_bounds}, the dual parameter $\lambda_k \in [0, 2/\delta]$, and $g\in\{r, c\}$.
    \begin{align*}
        (a)  \| \nabla_\theta J_g(\theta_k)\| \leq 6 \tau _{\mathrm{mix}} G_1 ,~~
        (b)  \| \nabla_\theta \mathcal{L}(\theta_k, \lambda_k)\| \leq (6+\frac{12}{\delta}) \tau _{\mathrm{mix}} G_1 
    \end{align*}
\end{lemma}
\begin{proof}
    Using the policy gradient theorem \eqref{eq_policy_grad_theorem}, we have the following relation.  
\begin{align}
        \nabla_{\theta} J_g(\theta_k)=\mathbb{E}_{(s, a) \sim \nu^{\pi_{\theta_k}}}\bigg[ Q_g^{\pi_{\theta_k}}(s,a)\nabla_{\theta}\log\pi_{\theta_k}(a|s)\bigg]
\end{align}
Applying Lemma \ref{lemma_aux_2}$(b)$ and Assumption \ref{assump:score_func_bounds}$(a)$, we get
\begin{align}
     \| \nabla_\theta J_g(\theta_k)\| =  \norm{\mathbb{E}_{(s, a)\sim \nu^{\pi_{\theta_k}}}\bigg[ Q_g^{\pi_{\theta_k}}(s,a)\nabla_{\theta}\log\pi_{\theta_k}(a|s)\bigg] } \leq 6 \tau _{\mathrm{mix}} G_1
\end{align}
Combining the above result with the definition of the Lagrange function and the bound that $\lambda_k \leq \frac{2}{\delta}$, we arrive at the following result.
\begin{align}
     \| \nabla_\theta \mathcal{L}(\theta_k, \lambda_k)\|  &= \| \nabla_{\theta} J_r(\theta_k) + \lambda_k \nabla_{\theta} J_c(\theta_k) \| \nonumber\\
     &\leq \| \nabla_{\theta} J_r(\theta_k)\| + \lambda_k \|\nabla_{\theta} J_c(\theta_k) \| \leq (6+\frac{12}{\delta}) \tau _{\mathrm{mix}} G_1
\end{align}
This concludes the proof of Lemma \ref{lemma_grad_JL_bound}.
\end{proof}

\begin{lemma}[Strong duality]\citep[Lemma 3]{ding2023convergence}
    \label{lem.duality}
    We restate the problem \eqref{eq:unparametrized_formulation} below for convenience.
    \begin{align}
        \label{eq:rewrite_unparameterized}
        \max_{\pi\in\Pi} ~& J_r^{\pi} 
        ~~\text{s.t.} ~ J_c^{\pi}\geq 0
    \end{align} 
    where $\Pi$ is the set of all policies. Define $\pi^*$ as an optimal solution to the above optimization problem. Define the associated dual function as
    \begin{align}
        J_D^{\lambda}\triangleq\max_{\pi\in\Pi} J_r^{\pi}+\lambda J_c^{\pi}
    \end{align}
    and denote $\lambda^*=\arg\min_{\lambda\geq 0} J_D^{\lambda}$. We have the following strong duality property for the unparameterized problem.
    \begin{align}
        \label{eq:duality}
        J_r^{\pi^*} = J_D^{\lambda^*} 
    \end{align}	
\end{lemma}

\begin{lemma}[Lemma 16, \cite{bai2024learning}]
    \label{lemma:boundness}
    Consider the parameterized problem \eqref{eq:objective} where $\Theta$ is the collection of all parameters. Under Assumption \ref{ass_slater}, the optimal dual variable, $\lambda^*_{\Theta}=\arg\min_{\lambda\geq 0}\max_{\theta\in \Theta} J_r^{\pi_{\theta}}+\lambda J_c^{\pi_{\theta}}$, for the parameterized problem can be bounded as follows.
    \begin{equation}
        0 \leq \lambda_\Theta^* \leq \frac{J_r^{\pi^*}-J_r(\bar{\theta})}{\delta}\leq \dfrac{1}{\delta}
    \end{equation}
    where $\pi^*$ is an optimal solution to the unparameterized problem \eqref{eq:rewrite_unparameterized} and $\bar{\theta}$ is a feasible parameter mentioned in Assumption \ref{ass_slater}. 
\end{lemma}

Notice that in Eq. \eqref{eq:update}, the dual update is projected onto the set $[0,\frac{2}{\delta}]$ because the optimal dual variable for the parameterized problem is bounded in Lemma \ref{lemma:boundness}. We note that our dual updating technique remains the same as \cite{bai2024learning}, however we were able to achieve an improvement of global convergence rate due to the use of natural policy gradient and a more prudent choice of stepsizes ($\alpha, \beta$).

\begin{lemma}
    \label{lem:bridge}
    Assume that the Assumption \ref{ass_slater} holds. Define $v(\tau) = \max_{\pi\in \Pi}\{J_r^{\pi}|J_c^{\pi}\geq \tau\}$ where $\Pi$ is the set of all policies. The following holds for any $\tau\in\mathbb{R}$ where $\lambda^*$ is the optimal dual parameter for the unparameterized problem as stated in Lemma \ref{lem.duality}.
    \begin{equation}
        v(0)-\tau\lambda^* \geq	v(\tau)
    \end{equation}
\end{lemma}

\begin{proof}
    Using the definition of $v(\tau)$, we get $v(0) = J_r^{\pi^*}$ where $\pi^*$ is a solution to the unparameterized problem \eqref{eq:rewrite_unparameterized}. Denote $\mathcal{L}(\pi,\lambda)=J_r^{\pi}+\lambda J_c^{\pi}$. By the strong duality stated in Lemma \ref{lem.duality}, we have the following for any $\pi\in\Pi$.
    \begin{equation}
        \mathcal{L}(\pi,\lambda^*)\leq \max_{\pi\in\Pi} \mathcal{L}(\pi,\lambda^*)\overset{Def}=J_D^{\lambda^*}\overset{\eqref{eq:duality}}=J_r^{\pi^*}=v(0)
    \end{equation}
    Thus, for any $\pi\in\{ \pi\in\Pi \,\vert\,J_c^{\pi} \geq \tau \}$, we can deduce the following.
    \begin{equation}
	\begin{aligned}
            v(0)-\tau\lambda^*\geq \mathcal{L}(\pi,\lambda^*)-\tau\lambda^*=J_r^{\pi}+\lambda^*(J_c^\pi-\tau) \geq J_r^{\pi}
	\end{aligned}
    \end{equation}
    Maximizing the right-hand side of this inequality over $\{ \pi\in\Pi \vert J_{c}^{\pi}\geq \tau \}$ yields
    \begin{equation}
        \label{eq.opt1}
	v(0)- \tau\lambda^* \geq v(\tau)
    \end{equation}
    This completes the proof of Lemma \ref{lem:bridge}.
\end{proof}

\begin{lemma}\label{lem.constraint}
    Let Assumption \ref{ass_slater} hold and $(\pi^*, \lambda^*)$ be the optimal primal and dual parameters for the unparameterized problem \eqref{eq:rewrite_unparameterized}. For any constant $C\geq 2\lambda^*$, if there exist a $\pi\in\Pi$ and $\zeta>0$ such that $J_r^{\pi^*}-J_r^{\pi}+C[-J_c^{\pi}]\leq \zeta$, then 
    \begin{equation}
        -J_c^{\pi}\leq 2\zeta/C
    \end{equation}
\end{lemma}

\begin{proof}
    Let $\tau = J_c^{\pi}$. We have the following inequality following the definition of $v(\tau)$ provided in Lemma \ref{lem:bridge}.
    \begin{equation}
        \label{eq.opt2}
	J_r^{\pi}\leq v(\tau)
    \end{equation}
     Combining Eq. \eqref{eq.opt1} and \eqref{eq.opt2}, and using the fact that $v(0)=J_r^{\pi^*}$, we have the following inequality.
    \begin{equation}
        J_r^{\pi}-J_r^{\pi^*}\leq v(\tau)-v(0)\leq -\tau\lambda^*
    \end{equation}
    Using the condition in the Lemma, we have
	\begin{equation}
            (C - \lambda^*) (-\tau) = {\tau} \lambda^*+C (-\tau)\leq J_r^{\pi^*}-J_r^{\pi}+C (-\tau)\leq \zeta
	\end{equation}
    Since $C-\lambda^*\geq C/2 >  0$, we finally have the following inequality. 
    \begin{equation}
        (-\tau)\leq \frac{\zeta}{C-\lambda^*}\leq\frac{2\zeta}{C}
    \end{equation}
    This completes the proof of Lemma \ref{lem.constraint}.
\end{proof}

\end{document}